\documentclass{article}

\usepackage{arxiv}

\usepackage{amsmath}
\usepackage{graphicx}
\usepackage[colorinlistoftodos]{todonotes}
\usepackage{paralist}
\usepackage{verbatim}
\usepackage{xspace}
\usepackage{booktabs}
\usepackage{amssymb}
\usepackage{amsthm}
\usepackage{nccmath}
\usepackage{graphicx}
\usepackage{lineno}
\usepackage{url}
\usepackage{xifthen}
\usepackage{tikz}
\usepackage{caption}
\usepackage{subcaption}
\usepackage{natbib}

\usepackage{pifont}
\newcommand{\yes}{\ding{52}}
\newcommand{\no}{\ding{56}}

\theoremstyle{definition}
\newtheorem{example}{Example}[section]
\newtheorem{theorem}{Theorem}
\newtheorem{definition}{Definition}

\usetikzlibrary{arrows,decorations,backgrounds,automata,shapes,patterns,decorations,backgrounds,plotmarks,calc}

\tikzstyle{every initial by arrow}=[initial text=]
\tikzstyle{every state}=[fill=none,draw=black,text=black]
\tikzstyle{every state}=[fill=none,draw=black,text=black,inner sep=0pt,minimum
size=7mm]
\tikzstyle{every picture}=[->,>=stealth',shorten >=1pt,auto,node distance=2.5cm,
semithick]
\tikzstyle{sim}=[->,dotted]

\def\citeA{\citet}
\def\cite{\citep}

\newcommand{\optarg}[1]{\ifthenelse{\isempty{#1}}{}{#1}}

\def\xor{\oplus}
\def\override{\Leftarrow}

\def\c{\emph{c}\xspace}

\def\P{\emph{P}\xspace}
\def\Q{\emph{Q}\xspace}
\def\nonevent{not-}

\def\notQ{\nonevent{\Q}\xspace}
\def\HistoryP{\emph{H$_\P$}\xspace}
\def\HistoryQ{\emph{H$_\Q$}\xspace}
\def\HistoryNOTQ{\emph{H$_{not\Q}$}\xspace}
\def\causes{\leadsto}
\def\G{$_\textrm{G}$\xspace}

\def\HP{HP\xspace}
\def\U{\mathcal{U}}
\def\V{\mathcal{V}}
\def\R{\mathcal{R}}
\def\S{\mathcal{S}}

\def\F{\mathcal{F}}
\def\u{\vec{u}}
\def\K{\mathcal{K}}

\newcommand{\vecvar}[2]{\vec{#2}_{#1}}
\newcommand{\W}[1][]{\vecvar{#1}{W}}
\newcommand{\X}[1][]{\vecvar{#1}{X}}
\newcommand{\Y}[1][]{\vecvar{#1}{Y}}
\newcommand{\Z}[1][]{\vecvar{#1}{Z}}
\newcommand{\w}[1][]{\vecvar{#1}{w}}
\newcommand{\x}[1][]{\vecvar{#1}{x}}
\newcommand{\y}[1][]{\vecvar{#1}{y}}
\newcommand{\z}[1][]{\vecvar{#1}{z}}
\def\Yay{\Y \leftarrow \y}
\def\Xax{\X \leftarrow \x}
\def\Xay{\X \leftarrow \y}

\def\Waw{\W \leftarrow \w}
\def\Zaz{\Z \leftarrow \z}
\newcommand{\Xex}[1][]{\X{#1} = \x{#1}}
\newcommand{\Yey}[1][]{\Y{#1} = \y{#1}}
\newcommand{\Xey}[1][]{\X{#1} = \y{#1}}
\newcommand{\Wew}[1][]{\W{#1} = \w{#1}}
\newcommand{\Zez}[1][]{\Z{#1} = \z{#1}}

\def\causalformula{[\Yay]\phi}

\newcommand{\Fx}[1][]{F^{\phi{#1}}}
\newcommand{\Fy}[1][]{F^{\psi{#1}}}

\title{Contrastive Explanation: A Structural-Model Approach}

\author{Tim Miller\\
School of Computing and Information Systems\\University of Melbourne, Melbourne, Australia\\\url{tmiller@unimelb.edu.au}}

\begin{document}

\maketitle


\begin{abstract}
This paper presents a model of contrastive explanation using structural casual models.
The topic of causal explanation in artificial intelligence has gathered interest in recent years as researchers and practitioners aim to increase trust and understanding of intelligent decision-making. While different sub-fields of artificial intelligence have looked into this problem with a sub-field-specific view, there are few models that aim to capture explanation more generally. One general model is based on \emph{structural causal models}. It defines an explanation as a fact that, if found to be true, would constitute an actual cause of a specific event. However, research in philosophy and social sciences shows that explanations are \emph{contrastive}: that is, when people ask for an explanation of an event -- the \emph{fact} --- they (sometimes implicitly) are asking for an explanation relative to some \emph{contrast case}; that is, ``{Why \P rather than \Q}?''. In this paper, we extend the structural causal model approach to define two complementary notions of \emph{contrastive explanation}, and demonstrate them on two classical problems in artificial intelligence: classification and planning. We believe that this model can help researchers in subfields of artificial intelligence to better understand contrastive explanation.
\end{abstract}




\section{Introduction}

\begin{quote}
The key insight is to recognise that one does not explain events per se, but that one explains why the puzzling event occurred in the target cases but not in some counterfactual contrast case --- \citeA[p. 67]{hilton1990conversational}.
\end{quote}

The recent explosion in research and application of artificial intelligence has seen a resurgence of \emph{explainable artificial intelligence} (XAI) --- a body of work that dates back over three decades; for example, see \cite{buchanan1984rule,chandrasekaran1989explaining,swartout1993explanation}. This resurgence is driven by lack of trust from users \cite{stubbs2007autonomy,linegang2006human,mercado2016intelligent},  and also concerns regarding the ethical and societal implications of decisions made by `black box' algorithms \cite{angwin2016machine}.

One key mode of XAI is explanation. An explanation is a justification or reason for a belief or action.  There has been a recent burst of research on explanation in artificial intelligence, particularly in machine learning. Much of the work in XAI has centred around extracting the causes (or main causes) of a decision or action. While finding causes is an important part of explanation, people do so much more when explaining complex events to each other, and we can learn much from considering how people generate, select, present, and evaluate explanations. 

\citeA{miller2017explanation} systematically surveyed over 250 papers in philosophy, psychology, and cognitive science on how people explain to each other, and noted perhaps the most important finding is that explanations are \emph{contrastive}. That is, people do not ask ``Why \P?''; they ask ``Why \P rather than \Q?'', although often \Q is implicit from the context. Following \citeA{lipton1990contrastive}, we will refer to \P as the \emph{fact} and \Q as the \emph{contrast case}. 

Researchers in social science argue that contrastive explanation is important for two reasons. First, people ask contrastive \emph{questions} when they are surprised by an event and expected something different. The contrast case identifies what they expected to happen \cite{hilton1990conversational,van2002remote,lipton1990contrastive,chin2017contrastive}. This provides a `window' into the questioner's mental model, identifying what they do not know \cite{lewis1986causal}. Second, giving contrastive \emph{explanations} is simpler, more feasible, and cognitively less demanding to both questioner and explainer \cite{lewis1986causal,lipton1990contrastive,ylikoski2007idea}. \citeauthor{lewis1986causal} argues that a contrastive question ``requests information about the features that differentiate
the actual causal history from its counterfactual alternative.'' \cite[p.\ 231]{lewis1986causal}.

\citeA{lipton1990contrastive} defines the answer to a contrastive question as the \emph{Difference Condition}:

\begin{quote}
  To explain why \P rather than \Q, we must cite a causal difference between \P and \notQ, consisting of a cause of \P and the absence of a corresponding event in the history of \notQ. -- \citeA[p.\ 256]{lipton1990contrastive}.
\end{quote}

Following this, the explainer does not need to reason about \emph{or even know about} all causes of the fact --- only those relative to the contrast case. 

As an example, consider an algorithm that classifies images of animals. Presented with an image of a crow, the algorithm correctly identifies this as a crow. When asked for a reason, a good \emph{attribution} would highlight features corresponding the crow: its beak, feathers, wings, feet, and colour --- those properties that correspond to the model of a crow. However, if the question is: ``Why did you classify this as a crow instead of a magpie?'', the questioner already identifies the image as a bird. The attribution that refers to the beak, feathers, wings, and feet makes a poor explanation, as a magpie also has these features. Instead, a good explanation would point to what is different, such as the magpie's white colouring and larger wingspan.

Importantly, the explanation fits directly within the questioner's `window' of uncertainty, and is smaller and simpler, even on this trivial example. AI models, though, are typically more complicated and more structured, implying that contrastive explanation can provide much benefit by adhering to Grice's conversational maxims \cite{grice1975logic} of \emph{quantity}: make your contribution as informative as is required, and do not make it more informative than is required; and \emph{relation}: only provide information that is related to the conversation.

Further to this, \cite{wachter2017counterfactual} argue that explanations using contrastive explanations can be used as a means for explaining individual decisions, providing sufficient explanatory power for individuals to understand and contest decisions, without imposing requirements of opening up the mechanism for decision making.



In this paper, we extend \citeauthor{halpern2005causes-part-II}'s definition of explanation using \emph{structural causal models} \cite{halpern2005causes-part-II} to the case of contrastive explanation, providing a general model of contrastive explanation based on \citeauthor{lipton1990contrastive}'s Difference Condition.
In particular, we define contrastive explanation for two types of questions: \emph{counterfactual} questions and \emph{bi-factual} questions. An counterfactual question is of the form ``Why \P rather than \Q?'', and asks why some fact \P occurred instead of some hypothetical \emph{foil} \Q. A bi-factual question is of the form ``Why \P but \Q?'', and asks why some fact \P occurred in the current situation while some \emph{surrogate} \Q occurred in some \emph{other} factual situation. The difference is that in the former, the foil is hypothetical, while in the latter, the surrogate is actual and we are contrasting two events that happened in different situations. From the perspective of artificial intelligence, the former is asking why a particular algorithm gave an output rather than some other output that the questioner expected, while the latter is asking why an algorithm gave a particular output this time but some (probably different) output another time. 

We define what it means to have a cause of these two contrastive questions, and what it means to explain them. Although it is not possible to prove such a model is `correct', we show that the model is internally consistent, and demonstrate it on two representative examples in artificial intelligence: classification and goal-directed planning. 

\section{Related Work}
\label{sec:lit}

\subsection{Philosophical Foundations}
\label{sec:lit:philosophical-foundations}

In the social sciences, it is generally accepted that explanations are contrastive \cite{miller2017explanation}. The questions that people ask have a contrast case, which is often implicit, and the explanations that people give explain relative to this contrast case. Even when giving an explanation with no question, people explain relative to contrast cases.

\citeA{garfinkel1981forms} seems to be the first to make a case for contrastive explanation\footnote{Although \citeA[p.\ 127]{van1980scientific} attributes the idea of contrastive explanation to Bengt Hannson in an unpublished manuscript circulated in 1974.}. He provides a story about a well-known bank robber Willie Sutton who purportedly replied to journalist who asked why he robbed banks, with: ``That's where the money is.'' \citeauthor{garfinkel1981forms} argues that Sutton answered why he robs [banks/other things], rather than why he [robs/does not rob] banks because he answered to a different contrast case: that of banks vs.\ non-banks, rather than robbing vs.\ not robbing. \citeauthor{garfinkel1981forms} notes that these two different contrasts create two different contexts, and that explanations are relative to these contrastive contexts. An object of explanation is not just a state of affairs, but a ``state of affairs together with a definite \emph{space of alternatives} to it'' \cite[p.\ 21]{garfinkel1981forms}.

At the same time, \citeA{van1980scientific} was also arguing the case of contrastive explanations. He states that the underlying structure of a why--question is: ``Why (is it the case that) \emph{P} \emph{in contrast to} (other members of) \emph{X}?'', in which \emph{P} is the \emph{topic} and \emph{X} is the \emph{contrast class} to \emph{P} \cite[p.\ 127]{van1980scientific}. An answer to such a question has the structure ``\emph{P} \emph{in contrast to} (the rest of) \emph{X} because \emph{A}'' \cite[p.\ 143]{van1980scientific}. \citeauthor{van1980scientific} argues that when a questioner asks such a question, they presuppose that: (1) the topic \emph{P} is true; (2) all other elements of the contrast class \emph{X} are false; and (3) \emph{A} is both true and explanatorily relevant  to the topic. He proposes an explicit relation \emph{R} that determines explanatory relevance.

\citeA{hesslow1983explaining,hesslow1988problem} extends this idea of explanatory relevance and seems to be the first to make a case for the idea of contrast cases themselves defining \emph{explanatory relevance}. He argues that there is a distinction between \emph{determining causes} and \emph{explanatory causes}, with the former being the (often large) set of conditions that contribute to causing an event, and the latter being a subset of the determining causes that are selected due to their \emph{explanatory power}.
\citeauthor{hesslow1988problem}'s theory of explanation is based on two complementary ideas. The first is that of contrastive explanation. He states that:

\begin{quote}
  \ldots the effect or the explanandum; i.e.\ the event to be explained, should be construed, not as an object's having a certain property, but as a \emph{difference} between objects with regard to that property. --- \citeA[p.\ 24]{hesslow1988problem}
\end{quote}

The second is of explanatory relevance. \citeauthor{hesslow1988problem} argues that by explaining only those causes that are different between the two or more objects, the explanation is more relevant to the questioner as it provides those causes that the questioner does not know. In essence, the contrast case provides a window into the particular causes that the questioner does not understand.

\citeauthor{hesslow1983explaining} presents an example: ``Why did the barn catch on fire?''. The explanation that someone dropped a lit cigarette in the hay has strong explanatory power and would satisfy most people. But what about other causes? The presence of oxygen, the hay being dry, and absence of fire sprinklers are all causes, but the cigarette has particular explanatory power because oxygen is always present in barns, and most barns are dry and have no fire sprinklers. The explanation is \emph{contrasting} to these normal cases.

He formalises this notion as follows. Given an object $a$, a property $E$, and a reference class $R$ (the contrast cases), the cause $Ca$ is an \emph{adequate explanation} of $\langle a, E, R\rangle$ iff:
\begin{enumerate}
	\item for all $x$ in $R$, if $Cx$ had been true then $Ex$ would have been true; and
	\item if $\neg Ca$ had been true, then $\neg Ea$ would have been true,
\end{enumerate}

\noindent
in which $Cx$ and $Ex$ refer to the cause $C$ and property $E$ respectively applying to $x$. This states that $Ca$ is an adequate explanation if and only iff (1) if the cause $C$ held on all the other objects $x$ in $R$ (e.g.\ other barns), then the property $E$ would also hold (the other barns would have also caught fire); and (2) if the cause $C$ did not apply to $a$, then the property $E$ would not hold. We can see that (1) does not apply to oxygen, because oxygen is present in other barns that do not catch fire, while for the cigarette this is the case; and that (2) applies to the cigarette --- if the cigarette had not been dropped, the fire would not have occurred.

At a similar time, \citeA{lewis1986causal} proposed a short account of contrastive explanation. According to \citeauthor{lewis1986causal}, to explain why \P occurred rather than \Q, one should offer an event in the history of \P that would not have applied to the history of \Q, if \Q had occurred. For example, he states: ``Why did I visit Melbourne in 1979, rather than Oxford or Uppsala or Wellington? Because Monash University invited me. That is part of the causal history of my visiting Melbourne; and if I had gone to one of the other places instead, presumably that would not have been part of the causal history of my going there'' \citeA[p.\ 229--230]{lewis1986causal}. This has parallels with \citeauthor{hesslow1988problem}'s account \cite{hesslow1983explaining,hesslow1988problem}.

\citeA{temple1988contrast} subsequently argued against the case of contrastive explanation. \citeauthor{temple1988contrast} argued that the question ``Why \P rather than \Q?'' presupposes that \P is true and \Q is not, and that the object of explanation is not to explain why \P and \Q are mutually exclusive, but instead to ask ``Why [\P and not \Q]?''. Therefore, contrastive why--questions are just standard propositional why--questions of the form ``Why \emph{X}?'', but with \emph{X} being [\P and \notQ].

However, \citeA{lipton1990contrastive} argues that this is a language phenomenon, and semantically, explaining ``{Why \P rather than \Q}?'' is not the same as explaining ``Why [\P and not \Q]?". Building on \citeauthor{lewis1986causal}'s interpretation based on the history of events \cite{lewis1986causal}, \citeauthor{lipton1990contrastive} argues that answering ``Why [\P and not \Q]?'' requires an explanation of \P and of \notQ. For example, to answer why the barn burned down rather than not burning down would require a complete attribution of why the barn burned down, including the presence of oxygen, as well as why other barns do not typically burn down. \citeauthor{lipton1990contrastive} argues that this is not what the explainee wants.

\citeA{lipton1990contrastive} proposes that explanation selection is best described using the \emph{Difference Condition}:

\begin{quote}
To explain why \P rather than \Q, we must cite a causal difference between \P and \notQ, consisting of a cause of \P and the absence of a corresponding event in the history of \notQ. --- \citeA[p.\ 256]{lipton1990contrastive}.
\end{quote}

This differs from the definition of contrastive explanation from \citeA{lewis1986causal} in that instead of selecting a cause of \P that is not a cause of \Q if \Q had occurred, we should explain the \emph{actual difference} between \P and \notQ; that is, we should cite a cause that is in the actual history of \P, and an event that did not occur in the actual history of \notQ.

We can formalise this as the following, in which $\causes$ is the causal relation, and \HistoryP and \HistoryNOTQ are the \emph{history} of \P and \notQ respectively, and \HistoryQ is the hypothetical history of \Q had it occurred:

\begin{center} 
\begin{tabular}{ll@{$~\land~$}l@{~~~where~~~}l@{$~\land~$}l}
\toprule
\citeauthor{lewis1986causal}       & $c \causes P$ & $c \not\causes Q$ & $c \in $~\HistoryP & $c \notin $~\HistoryQ\\
\citeauthor{lipton1990contrastive} & $c \causes P$ & $c' \causes $~\Q & $c \in $~\HistoryP & $c' \notin $~\HistoryNOTQ\\
\bottomrule
\end{tabular}
\end{center}

Thus, \citeauthor{lewis1986causal}'s definition \cite{lewis1986causal} cites some alternative history of facts in which \Q occurred, whereas \citeauthor{lipton1990contrastive}'s definition \cite{lipton1990contrastive} refers to the \emph{actual} history of \notQ. Further, \citeauthor{lewis1986causal}'s definition states that the explanation should be \emph{an} event \c (or perhaps set of events), whereas \citeauthor{lipton1990contrastive}'s states that the explanation is the difference between \c and \c'.

It was generally accepted at the time that \citeA{lipton1990contrastive} proposed his ideas, that facts and contrast case are incompatible events \cite{temple1988contrast,garfinkel1981forms,van1980scientific,ruben1987explaining}; for example, a barn cannot both burn down and not burn down, or leaves cannot be blue and yellow at the same time. However, \citeauthor{lipton1990contrastive} notes that compatible contrastive cases are also valid. For example, we can ask why one leaf is blue while another leaf is yellow. It is perfectly possible that both leaves could be blue, but we are looking for explanations as to why only one of them is.

\citeA{ylikoski2007idea} provides a more refined model to explain this, noting that incompatible vs.\ compatible contrast cases are two different types of question. The first is when we contrast two incompatible contrasts of the same process; one the fact and one the `imagined' \emph{foil}, such as one a leaf being yellow instead of blue. The fact and the foil must be inconsistent. The second is when we contrast two facts from two actual and different processes. That is, both facts actually occurred, such as one yellow leaf and one blue leaf. \citeauthor{ylikoski2007idea} calls the second fact a \emph{surrogate} for a counterfactual claim about the first process. He claims that the surrogate is used to simplify the explanation --- as one simply needs to find the difference between the fact and surrogate, which is consist with the idea from \citeA{lipton1990contrastive} that this is cognitively a simpler problem. 
 
\citeA{van2002remote} divide explanatory questions into four types:

{\centering
	
	\vspace{2mm}
	
	\noindent
	\begin{tabular}{lp{11.5cm}}
		\emph{Plain fact}: &  Why does object $a$ have property $P$?\\
		\emph{P-contrast}: & Why does object $a$ have property $P$, rather than property $Q$?\\
		\emph{O-contrast}: & Why does object $a$ have property $P$, while object $b$ has property $Q$?\\
		\emph{T-contrast}: &  Why does object a have property $P$ at time $t$, but property $Q$ at time $t'$?
	\end{tabular}
}

This defines three types of contrast: within an object (P-contrast), between objects themselves (O-contrast), and within an object over time (T-contrast). P-contrast is the standard `rather than' interpretation, while O-contrast and T-contrast correspond to  \citeauthor{ylikoski2007idea}'s notion of different processes \cite{ylikoski2007idea}.

In Section~\ref{sec:why-questions}, we will formalise the notion of contrastive questions using the framework of \citeA{halpern2005causes-part-I}, and will show that the reasoning of \citeA{ylikoski2007idea} is natural with respective to structural equations and fits the types of questions we would expect in explainable artificial intelligence. The concept of \emph{P-contrast} is captured as \emph{counterfactual} explanations, while \emph{O-contrast} and \emph{T-contrast} are captured as \emph{bi-factual} explanations.

\subsection{Computational Approaches}
\label{sec:lit:computational}

In artificial intelligence, contrastive questions are not just a matter of academic interest. User studies investigating the types of questions that people have for particular systems identify ``Why not?'' questions and contrast classes as important. \citeA{lim2009assessing} showed that ``Why not?'' questions are important in context-aware applications, while \citeA{haynes2009designs} found that users of their virtual aviation pilot system particularly sought information about contrast cases. Given that this is consistent with views from philosophy and psychology, it makes sense to consider the difference condition as key to answering these questions.

The idea of why-not questions in artificial intelligence was around prior to these studies. The explanation module of the MYCIN expert system explicitly allowed users to pose questions such as ``Why didn't you do \emph{X}?'' \cite{buchanan1984rule}, which is providing a foil for the fact. More recently, there has been a keen interest in answering why-not questions for many different sub-fields of artificial intelligence, including machine learning classification \cite{dhurandhar2018explanations,mothilal2020explaining}, belief-desire-intention agents \citeauthor{winikoff2017debugging}, reinforcement learning \cite{madumal2020explainable,waa2018contrastive}, classical planning \cite{krarup2019model,sreedharan2018hierarchical}, and image classification \cite{akula2020cocox}, to cite just a few papers.

However, while the idea of why-not questions is being addressed, there is not a clear and consistent understanding of contrastive explanation in the explainable AI community. First, while there many of papers what answer why-not questions, most of these are \emph{counterfactual} solutions, not \emph{contrastive}, meaning that they to not make explicit use of the \emph{difference} between the fact and the foil, with some exceptions; for example, \citeA{madumal2020explainable}. This may seem trivial, and from an algorithmic perspective perhaps it is. But from the perspective of helping people to understand models and decisions, this is a crucial step that requires further research in human factors and human-computer interaction.

Second, there is no general understanding in the community about contrastive and counterfactual explanation, with much of it built on the intuitions of authors. This leads to algorithms that can be useful, but terminology and solutions that are not aligned. For example, \citeA{dhurandhar2018explanations} use the term \emph{pertinent negatives/positives} to refer to foils, while \citeA{akula2020cocox} use the term \emph{fault lines}, and \cite{krarup2019model} use \emph{foil}. 

Third, while there is work on counterfactual explanation that contributes to contrastive explanation, there is no work that the author is aware of that addresses why we call \emph{bi-factual} questions: why $P$ this time but $Q$ last time? As noted by \citeA{wang2019designing} in their study of clinicians using AI-driven diagnosis in intensive care units, some clinicians ``wanted to see past patients which had similar presentations (e.g., complaints, vital signs), but not necessarily similar diagnoses (decision outcomes)". This is a request for a bi-factual contrastive explanation. Showing similar similar cases with different outcomes, and importantly, showing the \emph{difference} between the two helps people to understand why the outcomes differ. Currently, this form of why-not question is lacking in explainable AI literature.

The aim of this paper is to mitigate some of the issues above by providing a \emph{general} model of contrastive explanation that can be mapped to other models in artificial intelligence, such as machine learning, planning, reinforcement learning, case-based reasoning, BDI agents, etc. By presenting a single, coherent general model, we can begin to understand the similarities and differences between proposed solutions and can start to plug holes, such as lack of bi-factual explanation, in explainable AI.

As far as the author is aware, \citeA{kean1998characterization} is the only other author to consider a general computational model of contrastive explanation. \citeauthor{kean1998characterization}'s model of contrastive explanation is also built on \citeauthor{lipton1990contrastive}'s Difference Condition \cite{lipton1990contrastive}. Given a knowledge base $\mathcal{K}$ and an observation $P$, \citeauthor{kean1998characterization} proposes a simple model to calculate why $P$ occurred instead of $Q$. \citeauthor{kean1998characterization} provides a definition of a \emph{non-preclusive} contrastive explanation for ``\emph{Why \P rather than \Q?}'', which refers to the propositions that are required for \P to hold but not \Q. The definition of a \emph{preclusive} contrastive explanation uses the Difference Condition, and, as in this paper, identifies that the contrastive explanation must reference both the causes of $\P$ as well as causes of $Q$ that were not true. There are three key differences between \citeauthor{kean1998characterization}'s model and the structural approach model approach in this paper. First, \citeauthor{kean1998characterization}'s model was published when the understanding of causality in artificial intelligence was in its infancy, and is therefore built on propositional logic, rather than on a logic of causality and counterfactuals, which is more suitable. Second, \citeauthor{kean1998characterization}'s model considers only `rather than' questions, and not contrastive explanations with surrogates rather than foils. Third,  and most importantly \citeauthor{kean1998characterization}'s model is in fact a model of abductive reasoning, in which assumptions are made about the truth of certain propositions to find the `best' explanation. As such, this is a model of the cognitive process of contrastive explanation from the perspective of an explainer, and the task is to derive an explanation for an observation.  There is no explainee in \citeauthor{kean1998characterization}'s model. In contrast, our model is not concerned with abductive reasoning, but instead models an explainer with complete knowledge of the explanation using the difference condition to communicate to an unaware explainee.
\section{Structural Models}
\label{sec:hp}

In this paper, we build definitions of contrastive questions and contrastive explanations based on \citeauthor{halpern2005causes-part-I}'s \emph{structural models} \cite{halpern2005causes-part-I}.  As opposed to previous models, which use logical implication or statistic relevance, Halpern and Pearl's definition is based on counterfactuals, modelled using \emph{structural equations}.

In Part~I \cite{halpern2005causes-part-I} of their paper, Halpern and Pearl provide a formal definition of causality. A \emph{causal model} is defined on two sets of variables: \emph{exogenous} variables, who values are determined by factors external to the model, and \emph{endogenous} variables, who values are determined by relationships with other (exogenous or endogenous) variables.

\subsection{Models}
Formally, a \emph{signature} $\S$ is a structure $(\U, \V, \R)$, in which  $\U$ is a set of exogenous variables, $\V$ a set of endogenous variables, and $\R$ is a function that defines the range of values for every variable $Y \in \U \cup \V$; that is, the range of a variable $Y$ is $\R(Y)$.

A \emph{causal model} is a pair, $M = (\S, \F)$, in which $\F$ defines a set of functions, one for each endogenous variable $X \in \V$, such that $F_X : (\times_{U\in\U}\R(U)) \times (\times_{Y\in \V-\{X\}}\R(Y)) \rightarrow \R(X)$ determines the value of $X$ based on other variables in the model. A causal model is said to be \emph{recursive} if it is acyclic.

A \emph{context}, $\u$, is a vector that gives a unique value to each exogenous variable $u \in \U$.  A model/context pair $(M, \u)$ is called a \emph{situation}.

\citeA{halpern2005causes-part-I} extend this basic structural equation model to support modelling of counterfactuals.
To represent counterfactual models, the model $M_{\Xax}$ defines the new causal model given a vector $\X$ of endogenous variables in $\V$ and their values $\x$ over the new signature $\S_{\X} = (\U, \V - \X, R|_{\V - \X})$. This represents the model $M$ with the values of $\X$ overridden by $\x$. Formally, this model is defined as $M_{\Xax} = (\S_{\X}, \F^{\Xax})$, in which each $F_Y^{\Xax}$ in $\F$ is defined by setting the values of $\X$ to $\x$ in function $F_Y$.

\subsection{Language}
To reason about these structures, in particular, counterfactuals, \citeA{halpern2005causes-part-I}  present a simple but powerful language. Given a signature $\S = (\U, \V, \R)$, variables $X \in \V$ and values $x \in \R(X)$, a formula of the form $X=x$ is called a \emph{primitive event}, and describes the event in which variable $X$ is given the value $x$. A \emph{basic causal formula} is of the form $[Y_1 \leftarrow y_1, \ldots, Y_n \leftarrow y_n]\phi$, in which $\phi$ is any Boolean combination of primitive events, each $Y_i$ is a variable in $\V$ (endogenous variable), and $y_i \in \R(Y_i)$. We will follow \citeauthor{halpern2005causes-part-I} in abbreviating this formula using $\causalformula$, in which $\Y$ and $\y$ are vectors of variables and values respectively. A \emph{causal formula} is a Boolean combination of basic causal formulas. If $\Y$ is empty, this is abbreviated as just $\phi$.

\subsection{Semantics}
Intuitively, a formula $\causalformula$ for a situation $(M, \u)$ states that $\phi$ would hold in the model if the counterfactual case of $Y_i = y_i$ for each $Y_i \in \Y$ and $y_i \in \y$ were to occur. More formally, Halpern and Pearl define $(M, \u) \models \phi$ to mean that $\phi$ holds in the model and context $(M, \u)$. The $\models$ relation is defined inductively by defining $(M, \u) \models [\Yay](X = x)$ as holding if and only if the unique value of $X$ determined from the model $M_{\Yay}$ is $x$, and defining Boolean combinations in the standard way.

\begin{example}
\label{example:arthropod-causal-model}
This section presents a simple example of a hypothetical system that classifies images of arthropods into several different types, taken from \citeA{miller2017explanation}. The categorisation is based on certain physical features of the arthropods, such as number of legs, number of eyes, number of wings, etc.  
Table~\ref{tab:arthropod} outlines a simple model of the features of arthropods for illustrative purposes.
\begin{table}[!th]
\centering
\begin{tabular}{lcccccc}
\toprule
 &  & & & \textbf{Compound} & \\[-1mm]
\textbf{Type} & \textbf{No.\ Legs} & \textbf{Stinger} & \textbf{No.\ Eyes} & \textbf{Eyes} & \textbf{Wings}\\

\midrule
Spider & 8 & \no  & 8 & \no  & 0\\
Beetle & 6 & \no  & 2 & \yes  & 2\\
Bee    & 6 & \yes & 5 & \yes  & 4\\
Fly    & 6 & \no  & 5 & \yes & 2\\
\bottomrule
\end{tabular}
\caption{A simple lay model for distinguishing common arthropods.}
\label{tab:arthropod}
\end{table}

The causal model for this has endogenous variables $L$ (number of legs), $S$ (stinger), $E$ (number of eyes), $C$ (compound eyes),  $W$ (number of wings), and $O$ (the output). $U_1$ is an exogenous variable that determines the actual type of the arthropod, and therefore causes the values of the properties such as legs, wings, etc. The variables $L$, $E$, and $W$ range over the natural numbers, while $S$ and $C$ are both Boolean. The output $O$ ranges over the set $\{Spider, Beetle, Bee, Fly, Unknown\}$. A causal graph of this is shown in Figure~\ref{fig:arthropod-causal-graph}.
The functions are clear from Table~\ref{tab:arthropod}; for example, $F_O(8, false, 8, no, 0) = Spider$, and $O=Unknown$ for anything not in the table.

\begin{figure}[!ht]
  \centering
  \begin{subfigure}[t]{0.48\textwidth}
    \vskip 0pt
    \centering
\begin{tikzpicture}
\begin{scope}[node distance=1.0cm and 0.25cm,font=\scriptsize]
\tikzstyle{empty state}=[state, fill=none,inner sep=0pt,opacity=0.0]

\node[empty state] (V) {};  //empty state so this aligns with other subfigure

\node[state] (O) [above left=of V]{$O$};
\node[state] (E) [above=of O]  {$E$};
\node[state] (S) [left=of E]  {$S$};
\node[state] (L) [left=of S]  {$L$};
\node[state] (C) [right=of E]  {$C$};
\node[state] (W) [right=of C]  {$W$};
\node[state] (U) [above=of E]  {$U_1$};

\path[every node/.style={sloped,anchor=south,auto=false,font=\tiny}]

(U) edge[]   node[sloped=true]  {$$}   (L)
(U) edge[]   node[sloped=true]  {$$}   (S)
(U) edge[]   node[sloped=true]  {$$}   (E)
(U) edge[]   node[sloped=true]  {$$}   (C)
(U) edge[]   node[sloped=true]  {$$}   (W)
(L) edge[]   node[sloped=true]  {$$}   (O)
(S) edge[]   node[sloped=true]  {$$}   (O)
(E) edge[]   node[sloped=true]  {$$}   (O)
(C) edge[]   node[sloped=true]  {$$}   (O)
(W) edge[]   node[sloped=true]  {$$}   (O)

;
\end{scope}

\end{tikzpicture}
    \caption{Causal graph for arthropod algorithm defined in       Example~\ref{example:arthropod-causal-model}}
    \vfill
    \label{fig:arthropod-causal-graph}
  \end{subfigure}
  \begin{subfigure}[t]{0.48\textwidth}
    \vskip 0pt
    \centering
\begin{tikzpicture}
\begin{scope}[node distance=1.0cm and 0.25cm,font=\scriptsize]
\tikzstyle{empty state}=[fill=none,inner sep=0pt,minimum size=2.5mm,opacity=0.0]

\node[state] (V) {$V$};

\node[state] (O) [above left=of V]{$O$};
\node[state] (E) [above=of O]  {$E$};
\node[state] (S) [left=of E]  {$S$};
\node[state] (L) [left=of S]  {$L$};
\node[state] (C) [right=of E]  {$C$};
\node[state] (W) [right=of C]  {$W$};
\node[state] (U1) [above=of E]  {$U_1$};

\node[state] (A) [right=of W]  {$A$};
\node[state] (U2) [above=of A]  {$U_2$};

\path[every node/.style={sloped,anchor=south,auto=false,font=\tiny}]

(U1) edge[]   node[sloped=true]  {$$}   (L)
(U1) edge[]   node[sloped=true]  {$$}   (S)
(U1) edge[]   node[sloped=true]  {$$}   (E)
(U1) edge[]   node[sloped=true]  {$$}   (C)
(U1) edge[]   node[sloped=true]  {$$}   (W)
(L) edge[]   node[sloped=true]  {$$}   (O)
(S) edge[]   node[sloped=true]  {$$}   (O)
(E) edge[]   node[sloped=true]  {$$}   (O)
(C) edge[]   node[sloped=true]  {$$}   (O)
(W) edge[]   node[sloped=true]  {$$}   (O)

(U2) edge[]   node[sloped=true]  {$$}   (A)
(O) edge[]   node[sloped=true]  {$$}   (V)
(A) edge[]   node[sloped=true]  {$$}   (V)

;
\end{scope}

\end{tikzpicture}
        \caption{Causal graph for extended arthropod algorithm defined in Example~\ref{example:extended-arthropod-causal-model}}
    \label{fig:extended-arthropod-causal-graph}
  \end{subfigure}
\end{figure}
\end{example}

It is important to note that SCMs like this can be used to model correlative machine learning algorithms. It is not the physical features of number of legs, eyes, etc., that cause a real spider to be a spider -- it is its genetic makeup that causes this. This genetic makeup causes both the spider's physical features and it to be a spider. However, here we model that a correlation is found between the variables represent physical features and the type, and the algorithm uses this to predict the arthropod type. However, the \emph{causal} direction of the prediction model is from physical features to arthropod type because this is the way the prediction is made. The variables representing the physical features cause the decision, even though this does not model the causes in the real world. Of course, a model that represents causes in the real world may offer a better explanatory model, but even in cases where we do not have one, a straightforward mapping from features to outputs fits into our model of contrastive explanation.

\begin{example}
	\label{example:extended-arthropod-causal-model}
	Consider a extension to the arthropod algorithm in Example~\ref{example:arthropod-causal-model} that verifies manual annotations on arthropod images. Images are labels with one of $Spider$, $Beetle$, $Bee$, $Fly$, or no label ($Unknown$), and the new algorithm extends the previous one to check whether the manual annotations are correct or not. The same categories exist, but some images are not labelled at all. To model this, we add a new exogenous variable $U_2$, which determines the new endogenous variable $A$ -- the annotation on the image. A second endogenous variable $V$ with domain $\{Pass, Fail\}$ determines whether the classifier output $O$ corresponds with $A$. The causal graph is shown in Figure~\ref{fig:extended-arthropod-causal-graph}. The function $F_V(O,A)=Pass$ if either $A=O$ or $A=Unknown$ or $O=Unknown$, to avoid too many false negatives.  Otherwise, $F_V(O,A)=Fail$.
\end{example}
\section{Contrastive `Why' Questions}
\label{sec:why-questions}

The basic problem of explanation is to answer a \emph{why--question}. According to \citeA{bromberger1966whyquestions}, a why--question is just a \emph{whether--question}, preceded by the word `why'. A whether--question is an interrogative question whose correct answer is either `yes' or `no'.  The \emph{presupposition} within a why--question is the fact referred to in the question that is under explanation, expressed as if it were true (or false if the question is a negative sentence). For example, the question ``\emph{why did they do that?}'' is a why-question, with the inner whether-question being ``\emph{did they do that?}'', and the presupposition being ``\emph{they did that}''. 

However, as discussed already, why--questions are structurally more complicated than this: they are \emph{contrastive}. The question then becomes: \emph{what is a contrastive why--question?}

In this section, we extend \cite{ylikoski2007idea}'s argument for the existence of (at least) two different types of contrastive why--questions \cite{ylikoski2007idea}. In brief, the first asks why some fact happened rather than some other thing, called the \emph{foil}, while the second asks why some fact happened in one situation while another fact, called the \emph{surrogate}, happened in another (presumably similar) situation. The first type we call `rather than' or \emph{counterfactual} explananda, because in this case, the foil is a counterfactual possibility to the fact.  Intuitively, the fact and the foil are incompatible: it is not possible that both of them could have occurred. This is consistent with \citeauthor{temple1988contrast}'s reading that \Q offers an ``exclusive alternative in the circumstances'' \cite{temple1988contrast}.  The second type, we call \emph{bi-factual} explananda, because both the fact and the surrogate events actually occurred, but just in different contexts. The explainee is using the surrogate as a reference point to contrast against the fact.  Using \citeauthor{halpern2005causes-part-I}'s structural models \cite{halpern2005causes-part-I}, we more crisply demonstrate why there is a difference between these two questions based on the relationships between the situations in which the fact and its contrast case (foil or surrogate) did and did not occur respectively.

\subsection{Counterfactual Explananda}

Given two events \P and \Q, \citeA{lipton1990contrastive} defines a contrastive why--question as:
\begin{equation}
\label{defn:contrastive-question-informal}
\textrm{Why \P rather than \Q?}
\end{equation}
For a counterfactual explananda, this means that, in some situation, the fact $\P$ occurred and the explainee is asking why foil $\Q$ did not occur in that situation instead. To semi-formalise this in structural models, a counterfactual why--question, given a situation $(M, \u)$,  is:
\begin{equation}
\label{defn:contrastive-question-one-situation}
\textrm{Why $(M, \u) \models \phi$ rather than $\psi$~?}
\end{equation}
in which $\phi$ is the fact and $\psi$ is the foil. This assumes that $\phi$ is actually true in the situation $(M,\u)$, and that $\psi$ is not. The linguistic reduction to ``\emph{Why \P and \notQ?}'' is:
\begin{equation}
\label{defn:contrastive-question-one-situation-counterfactual}
\textrm{Why $(M, \u) \models \phi \land \neg\psi$?},
\end{equation}
%

To answer the question in Equation~\ref{defn:contrastive-question-one-situation-counterfactual}, one could argue that an explanation of such a case is a proof of $\phi$ and a counter-example for $\psi$. However, as argued by \citeA{lipton1990contrastive}, this is not really what is asked by ``\emph{Why $\phi$ rather than $\psi$?}. The `rather than' is asking for a relationship between the causes of $\phi$ and the causes (or non-causes) of $\psi$. As a counterexample to the reductionist argument, \citeauthor{lipton1990contrastive} notes that we can answer a `rather than' question without knowing all causes of the events. For instance, take the arthropod description from Example~\ref{example:arthropod-causal-model}, and a question as to why the algorithm classified a particular image as a Bee rather than a Fly. Assume that we only know the value of one variable in the model: $W$ --- the number of wings. We cannot give the cause of $O=Bee$ if we do not know the values of the other variables\footnote{Although in this trivial example, technically we could infer them all, but this is a property of the particular example, not of `rather than' questions and structural models in general.}. However, we can still give a perfectively satisfactory answer to the question: it is a Bee rather than a Fly because it has four wings instead of two. As such, `rather than' questions must be asking something different to just ``\emph{Why $\phi$ and why $\neg\psi$?}'', for which we need to know all causes for both $\phi$ and $\psi$.

These counterfactual explananda make sense as why--questions in artificial intelligence. Given the arthropod classification example, a `rather than' question represents an observer asking  why the output was a particular arthropod rather some other incompatible foil case; which would presumably often be the answer they were expecting.

In this paper, we assume that in counterfactual explananda, $\phi$ and $\psi$ are incompatible. It is clear that questions such as ``{Why $x \leq 5$ rather than $x \geq 0$}, where $x=4$ and therefore both fact and foil are true, do not make sense. However, one could argue that it is possible to ask `rather than' questions with compatible fact and foils over different variables; for example ``{Why $x=4$ rather than $y=5$}?''. It is not difficult to find a structural model such that $x=4$ and $y=5$. However, the value of $y$ in the actual situation must be something other than $5$, otherwise the question does not make sense. So, the question is really ``{Why $x=4 \land y=4$ rather than $x=4 \land y=5$}?'', which is incompatible. For this reason, we make the reasonable assumption that $\phi$ and $\psi$ always refer to the same variables and they are incompatible in the given situation. 

\subsection{Bi-factual Explananda}
%
As outlined in Section~\ref{sec:lit}, \citeA{ylikoski2007idea} argues that some contrastive why--questions can have compatible facts and foils; although he terms a compatible foil as a \emph{surrogate}. To be compatible, he argues that they must occur as part of two different `processes'.

We model this second type of contrastive question, called a \emph{bi-factual} explananda, by modelling the two different processes as two different situations:
\begin{equation}
\label{defn:contrastive-question-two-situations}
\textrm{Why $(M, \u) \models \phi$ but $(M', \u') \models \psi$?}
\end{equation}
in which the $(M, \u)$ and $(M', \u')$ are two different situations, including two different models $M$ and $M'$, $\phi$ is the fact, and $\psi$ is the surrogate. Note the absence of `rather than' in the question. Linguistically, this makes sense because both the fact and the surrogate are actual --- there is no hypothetical case.

As a question in explainable AI, this question has a clear interpretation that $M$ and $M'$ refer to two different algorithms and $\u$ and $\u'$ define different `inputs' to the algorithms.
For the arthropod example, a valid question is why the algorithm produced the output $\phi$ for input image $J$, while some previous execution of the algorithm produced the different output $\psi$ for different image $K$ (note that $M=M'$ in this example). The observer is trying to understand why the outputs were different, when she expected $\phi$ to be $\psi$ like it was in a previous instance. Another example is the case noted by \citeA{wang2019designing}, discussed in Section~\ref{sec:lit:computational}, of clinicians wanting to compare similar cases with different outcomes.

In the case where $M \neq M'$, an example is in which model $M'$ is an updated version of $M$ --- for example, new data has been feed into a learning approach to produce a more refined model ---, and the explainee is asking for why the result has changed between the two models, potentially with $\u = \u'$.

Although not naturally worded as a `rather than' question, it could be argued that the question is actually a `rather than' question in which the person is asking ``{Why $\phi$ this time and $\psi$ last time \emph{rather than} $\phi$ (or $\psi$) both times?}'':
\begin{equation}
\label{defn:contrastive-question-two-situations-rather-than}
\begin{gathered}
\textrm{Why $(M, \u) \models \phi$ and $(M', \u') \models \psi$}\nonumber\\
\textrm{~rather than~}\\
\textrm{$(M, \u) \models \phi$ and $(M', \u') \models \phi$~~ or~~ $(M, \u) \models \psi$ and $(M', \u') \models \psi$~?}\nonumber
\end{gathered}
\end{equation}
If we reduce this using the template of ``\emph{\P and \notQ}'', and simplify, the result is:
\begin{equation}
\begin{gathered}
\label{defn:contrastive-question-two-situations-rather-than-simplified}
\textrm{Why $(M, \u) \models \phi \land \neg\psi$ but $(M', \u') \models \psi \land \neg\phi$~?}
\end{gathered}
\end{equation}

This is just the same as the question in Equation~\ref{defn:contrastive-question-two-situations}, however, it assumes that the fact and surrogate are incompatible. This assumption is too strong, because a perfectly valid question is why two different situations are producing the \emph{same} outcome, despite the differences in the situation: the explainee expects the two outcomes to be different and wants an explanation as to why they are the same.

In this section, we have demonstrated a case for two types of contrastive why--question: counterfactual and bi-factual explananda. In the remainder of the paper, we use structural causal models to define what answers to these questions look like, starting with how to define \emph{contrastive cause} (Section~\ref{sec:contrastive-cause}) and then \emph{contrastive explanation} (Section~\ref{sec:contrastive-explanation}).

\section{Contrastive Cause}
\label{sec:contrastive-cause}

Before we turn to contrastive explanation, we define \emph{contrastive cause}. Explanations typically cite only a subset of the actual causes of an event, and research shows that various different criteria are used to select these, such as their abnormality, or epistemic relevance; see \citeA{miller2017explanation} for a discussion of these. In Section~\ref{sec:contrastive-explanation}, we build on the definition of explanation based on epistemic relevance by \citeA{halpern2005causes-part-II}. However, to do this, we first need to define what a contrastive cause is. 

Informally, a contrastive cause between $\phi$ and $\psi$ is a pair, in which the first element is a cause of $\phi$ and the second element is a cause of $\psi$. Intuitively, a contrastive cause $\langle A, B\rangle$ specifies that $A$ is a cause of $\phi$ that does not cause $\psi$, while $B$ is some corresponding event that causes $\psi$ but does not cause $\phi$. This is consistent with existing philosophical views; e.g.\ \citeA{ruben1987explaining} defines contrastive explanations as conjunctions between history of the contrasting events. The particular definition depends whether the why--question is counterfactual or bi-factual.

\subsection{Non-contrastive Cause}

Our definition of contrastive cause extends \citeauthor{halpern2005causes-part-I}'s definition of \emph{actual cause} \cite{halpern2005causes-part-I}. In their definition, causes are conjunctions of primitive events, represented as $\Xex$, while the events to be described are Boolean combinations of primitive events. 

\citeA{halpern2005causes-part-I} define two types of cause: \emph{sufficient cause} and \emph{actual cause}. Intuitively, a sufficient cause of an event in a situation is a conjunction of primitive events such that changing the values of some variables in that conjunct would cause the event not to occur. An actual cause is simply a minimal sufficient cause; that is, it contains no unnecessary conjuncts.

More formally, the conjunction of primitive events $\Xex$ is an actual cause of event $\phi$ in a situation $(M, \u)$ if the following three properties hold:

\begin{description}
	
	\item [AC1] $(M, \u) \models \Xex \land \phi$ --- that is, both the event and the cause are true in the actual situation.
	
	\item [AC2] There is a set $\W \subseteq \V$ and a setting $\x'$ of variables $\X$ such that if $(M, \u) \models \W = \w$ then $(M, \u) \models [\Xax', \Waw]\neg \phi$ --- that is, if $\X$ did not have the values $\x$ and all variables in $W$ remain the same, then event $\phi$ would not have occurred\footnote{Note that this is the later definition from \citeA{halpern2015modification}, which is simplified compared to the original definition of \citeA{halpern2005causes-part-I}. Halpern argues this updated definition is more robust.}.
	
	
	\item [AC3] $\X$ is minimal; no subset of $\X$ satisfies AC1 and AC2 -- that is, there are no unnecessary primitive events in the conjunction $\Xex$.
	
\end{description}

A \emph{sufficient cause} is simply the first two items above --- that is, a non-minimal actual cause.

Throughout the rest of this paper, we use the term \emph{partial cause} to refer to a subset of conjunctions of an actual cause.

\begin{example}
	Consider the arthropod example from Example~\ref{example:arthropod-causal-model}. $L=6$ (6 legs) is an actual cause of $O=Bee$ under the situation $u_3$ corresponding to line 3 of Table~\ref{tab:arthropod}. AC1 holds trivially because $L=6$ is in $u_3$ and $O=Bee$ is the output. AC2 holds because whenever $L \neq 6$, $O=Bee$ would not hold under $u_3$. AC3 holds because $L$  is just one variable, so is minimal. Similarly, all other `input' variables are actual causes in $u_3$; e.g.\ $E=6$.
\end{example}

\begin{example}
For the extended model with annotated images from Example~\ref{example:extended-arthropod-causal-model}, consider the situation $u_u$ in which there is no annotation ($A=Unknown$) and we have spider but with 7 legs ($L=7$). If $L=7$, then $O=Unknown$ and therefore the verification will pass ($V=Pass$), because this does not indicate an inconsistency.

One actual cause for $V=Pass$ is the pair $(L=7,A=Unknown)$. AC1 holds trivially. For AC2, we need to change both $L$ and $A$ to also change the value of $V$ to $Fail$. If we change $L$ to anything else, $V$ will remain $Pass$ because $A=Unknown$, and similarly if we change $A$. It requires a mismatch in $A$ and $O$ other than $Unknown$ to produce $V=Fail$. AC3 holds because the pair of $L$ and $O$ is minimal. Similarly, the pair $(O=Unknown, A=Unknown)$ is an actual cause. However, the triple $(L=7,A=Unknown,O=Unknown)$ is only a sufficient cause, because it is not minimal (violates AC3): we do not require both $L=7$ and $O=Unknown$.
\end{example}

\subsection{Contrastive Causes for Counterfactual Explananda}

To define contrastive cause, we adopt and formalise \citeauthor{lipton1990contrastive}'s Difference Condition \cite{lipton1990contrastive}, which states that we should find causes that are different in the `history' of the two events.
We define the `history' as the situation $(M,\u)$ under which the events are evaluated; that is, $(M, u)$ for counterfactual why--questions, and both $(M, u)$ and $(M',\u')$ for bi-factual why--questions. 

The particular explanandum for which we want to define cause is no longer a single event $\phi$, but a pair of events $\langle \phi, \psi \rangle$, in which $\phi$ is the fact and $\psi$ is the foil. Similarly, causes will consist of two events instead of one, consistent with the difference condition.

Informally, a contrastive counterfactual cause of a pair of events $\langle \phi, \psi\rangle$ is a pair of \emph{partial} causes, such that the difference between the two causes is the minimum number of changes required to make $\psi$ become true.

\begin{definition}[Contrastive Counterfactual Cause]
\label{defn:contrastive-counterfactual-cause}
A \emph{pair} of events $\langle \Xex, \Xey\rangle$ is an \emph{contrastive counterfactual actual cause} (also just a \emph{counterfactual cause}) of $\langle \phi,\psi\rangle$ in situation  $(M, \u)$ if and only if the following conditions holds:

\begin{description}

  \item [CC1] $\Xex$ is a partial cause of $\phi$ under $(M,\u)$.

  \item [CC2] $(M,\u') \models \neg\psi$ --- the foil $\psi$ is not true.

  \item [CC3]  There is a non-empty set $\W \subseteq \V$ and a setting $\w$ of variables in $\W$ such that $\Xey$ is a partial cause of $\psi$ under situation $(M_{\Waw}, \u)$.

 Informally, this states that there is some hypothetical situation that did not happen, but is feasible in $M$; and that $\Xey$ is a partial cause of $\psi$ under this hypothetical situation.

 \item [CC4] $(\Xex \cap \Xey) = \emptyset$ --- that is, there are no common events. This is the \emph{difference condition}. 

 \item [CC5] $\X$ is maximal --- that is, no superset of $\X$ satisfies CC1-4.

\end{description}
\end{definition}

Similar to the \HP definition, we can define \emph{sufficient} contrastive cause by modifying CC1 and CC3 to refer to partial sufficient causes.

This definition is based on the \citeA{halpern2015modification} definition of actual cause, as conditions CC1-3 directly access partial causes, which are subsets of actual causes. However, the definition is modular with respect to the underlying definition of actual cause, such that a different definition of actual cause (using structural models), such as the original definition from \citeA{halpern2005causes-part-I}, could be substituted, and this would change the semantic interpretation of the above.

The reader may expect to see that CC2 had an additional statement that no part of the hypothetical cause of $\psi$ is true, such as $\bigwedge_{X_i=y_i \in \Xey}X_i\neq y_i$. However, this is implied by CC4, because all elements of $\Xex$ are true, and each element of $\Xey$ is different from its corresponding value in $\Xex$. Also note that condition CC3 implies that the foil $\psi$ is feasible in $M$. That is, it implies that $M \not\models \neg\psi$. For an infeasible event, there cannot be another situation such $\Xey$ is a cause of $\psi$,  therefore there can be no difference condition. This seems reasonable though: asking why an infeasible foil did not occur should not invoke a difference between the fact and foil, but a description that the foil is infeasible.

\begin{example}
\label{example:arthropod-counterfactual-case}
Consider the arthropod example from Example~\ref{example:arthropod-causal-model}, asking why an image was categorised as a Bee instead of a Fly. To answer the counterfactual why--question, we take the maximal intersection of two actual causes of $Output = Bee$ and the hypothetical cause of $Output = Fly$. In this case, the following pairs correspond to the possible contrastive causes:
\begin{displaymath}
\begin{array}{l}
  \langle S=\ $\yes$, S=\ $\no$\rangle \\
  \langle W=4, W=2\rangle.
\end{array}
\end{displaymath}
The image was classified as a Bee instead of a Fly because the image contains a stinger (S) and four wings (W), while for a Fly, it would have required no stinger and two wings.
The other actual causes of $\phi$ and $\psi$, such as $L=6$, are not contrastive causes because they do not satisfy the difference condition in CC4.
\end{example}

It is difficult to argue that a particular definition of contrastive cause is correct. However, we can at least argue that they abide by some commonly-accepted properties; specifically, the properties of an \emph{adequate} explanation defined by \citeA{hesslow1983explaining} (see Section~\ref{sec:lit:philosophical-foundations}). This states that, if the counterfactual causes hold in a different situation, so to would the counterfactual events. The following theorem captures this.

\begin{theorem}
	\label{thm:cac-consistent}
	If $\mathcal{C}$ comprises \emph{all} counterfactual contrastive actual causes of $\langle\phi,\psi\rangle$ under situation $(M, \u)$, then for any maximal-consistent subset\footnote{We abuse notation slightly here: $\Xex$ is the conjunction of the first items of all of the subset; similarly $\Xey$ is the conjunction of the second items.} $\langle \Xex, \Xey\rangle \subseteq \mathcal{C}$:
	
	\begin{enumerate}[(a)]
		\item $(M, \u) \models [\Xay]\psi$; and
		\item $(M, \u) \models [\Xay]\neg\phi$.
	\end{enumerate}

We need to consider only the maximal-consistent subsets because the set of all contrastive causes could be inconsistent if there are multiple sufficient causes.
\end{theorem}
\begin{proof}
	Consider part (a) first. We prove via contradiction. Assume that $(M, \u) \not\models [\Xay]\psi$. From CC3, $\Xey$ contains partial causes of $\psi$, so there must be a set of additional causes $\Zez$ such that $(M, \u) \models [\Xay,\Zaz]\psi$. This implies that there is some (maximal) subset $\Zez['] \subseteq \Zez$ such that $(M,\u) \not\models \Zez[']$, and is therefore not in $\Xex$. However, these  two implications  mean that CC3 and CC4 hold for $\Zez[']$. CC5 also holds because $\Zez[']$ is maximal. Therefore, $\Zez[']$ is (one half of) a contrastive cause for $\langle \phi, \psi\rangle$, and as such, must be part of $\mathcal{C}$. Because $\Xey$ is maximal, $\Zez[']$ must be in $\Xey$, so it is not possible that both $(M,\u) \not\models [\Xay]\psi$ and $(M, \u) \models [\Xay,\Zaz]\psi$ are true. This contradiction shows that part (a) holds. Part (b) holds directly because $\phi$ and $\psi$ are incompatible.
\end{proof}

\subsection{Contrastive Causes in Bi-factual Explananda}

For bi-factual explananda, the definition of `history' is different to that of counterfactual explananda, citing two different situations.
We define the `history' as the situations $(M,\u)$ of $\phi$ and $(M', \u')$ of $\psi$.  For the moment, we simplify this by assuming that the two causal models $M$ and $M'$ are the same; e.g.\ the same algorithm is executed with different inputs from the environment. We drop this assumption later.

\begin{definition}[Contrastive Bi-factual Cause --- Simple Case]
	\label{defn:bi-factual-cause-simple}
A pair of events $\langle \Xex, \Xey\rangle$ is a \emph{contrastive bi-factual actual cause} of $\langle\phi, \psi\rangle$ in their respective situations  $(M, \u)$ and $(M,\u')$ if:

\begin{description}

 \item [BC1] $\Xex$ is a partial cause of $\phi$ under $(M,\u)$. 
 
 \item [BC2] $\Xey$ is a partial cause of $\psi$ under $(M,\u')$.

 \item [BC3] $(\Xex) \cap (\Xey) = \emptyset$ --- that is, there are no common events. This is the \emph{difference condition}.

 \item [BC4] $\X$ is maximal --- that is, no superset of $\X$ satisfies BC1-3.

\end{description}

\end{definition}

\noindent
Note that BC1 implies $(M, \u) \models \Xex \land \phi$ (AC1) and similarly for BC2.




A \emph{sufficient} contrastive cause can be obtained by modifying BC1 and BC2 to refer to partial sufficient causes.

This definition is simpler than that of counterfactual explanation (compare CC3 with BC2), because both the fact and surrogate are actual events, whereas in counterfactual explananda, the foil is hypothetical.

\begin{example}
\label{example:arthropod-bi-factual-case}
Consider again the arthropod example from Example~\ref{example:arthropod-causal-model}, and the contrastive why--question for two images $B$ and $F$, in which $B$ was categorised as a Bee and $F$ a fly. The situations for these two cases are straightforward to extract from Table~\ref{tab:arthropod}, as are the causes. To answer the contrastive why--question, we take the maximal intersection actual causes of $Output = Bee$ and $Output = Fly$ under models $(M, \u_{B})$ and $(M, \u_{F})$ respectively, which is simply the same as in Example~\ref{example:arthropod-counterfactual-case}:
\begin{displaymath}
\begin{array}{ll}
  \langle S=\ $\yes$, S=\ $\no$\rangle & \\
  \langle W=4, W=2\rangle.
\end{array}
\end{displaymath}
Note that the difference condition is the same as is in the counterfactual case, however, in this case, there was no need to find a hypothetical situation for the foil.
\end{example}

\begin{theorem}
\label{thm:BC-consistent}
If $\mathcal{C}$ comprises \emph{all} counterfactual contrastive actual causes of
$\langle\phi,\psi\rangle$ under respective situations $(M, \u)$ and $(M,\u')$
then for any maximal-consistent subset $\langle \Xex, \Xey\rangle \subseteq \mathcal{C}$:

\begin{enumerate}[(a)]
  \item $(M, \u)  \models [\Xay]\psi$; and
  \item $(M, \u')  \models [\Xax]\phi$.
\end{enumerate}
\end{theorem}
\begin{proof}
The proofs for both parts are similar to the proof for counterfactual causes in Theorem~\ref{thm:cac-consistent}, except that we refer to BC2-4 instead of CC3-5.
\end{proof}

Now we return to the case in which the two models may be different. For this, we define a \emph{restricted  cause} of $\phi$ under situation $(M, \u)$, where $M = (\S,\F)$, as a pair  $(\F', \Xex')$, in which $\Xex'$ is a sufficient cause of $\phi$, and $\F' \subseteq \F$ is the smallest subset of $\F$ required to derive $\phi$. That is, for all $\u'$, $(M,\u') \models \phi$ iff $(M^{\F'},\u') \models \phi$, where $M^{\F'}=(S,\F')$, and therefore all functions $F \setminus F'$ do not influence $\phi$ in any situation.
A \emph{partial restricted cause} is simply $(\Fx, \Xex)$ such that $\Fx \subseteq \F'$ and $\Xex \subseteq \Xex'$.

\begin{definition}[Contrastive Bi-factual Cause --- General Case]
A pair $\langle (\Fx, \Xex), (\Fy, \Yey)\rangle$ is a \emph{contrastive bi-factual actual cause} of $\langle\phi, \psi\rangle$ in their respective situations  $(M, \u)$ and $(M', \u')$ if and only if the following conditions hold:

\begin{description}

  \item [BC1\G] $(\Fx, \Xex)$ is a partial restricted cause of $\phi$ under situation $(M, \u)$.

  \item [BC2\G] $(\Fy, \Yey)$ is a partial restricted cause of $\psi$ under situation $(M', \u')$.

 \item [BC3\G] $\Fx \cap \Fy = \emptyset$ and $(\Xex) \cap (\Yey) = \emptyset$ --- that is, there are no common functions or pairs of events. This is the \emph{difference condition}. 

 \item [BC4\G] $(\Fx, \Xex, \Fy, \Yey, \X \cap \Y)$ is maximal.

  That is, there is no tuple ($\Fx['], \Xex['], \Fy['],  \Yey['], \X' \cap \Y') \neq (\Fx, \Xex, \Fy, \Yey, \X \cap \Y)$ satisfying BC1\G, BC2\G, and BC4\G such that $\Fx \subseteq \Fx[']$, $\Fy \subseteq \Fy[']$, $\Xex \subseteq \Xex[']$, $\Yey \subseteq \Yey[']$, and $\X \cap \Y \subseteq \X' \cap \Y'$.

\end{description}
\end{definition}

Note two differences between this and the less general version. First, the definition refers to differences in the functions of the two models. Second, the sets of variables that are referred to are no longer \emph{shared} between the two models. That is, in the less general definition, the contrastive cause $\langle \Xex, \Xey \rangle$ both pointed to $\X$. However, the sets of variables can be different between the two models $M$ and $M'$, so it is possible that some variable in $\Y$ does not exist in model $M$, but is a cause of $\psi$ in $(M',\u')$. Note that BC4\G also states that $\X \cap \Y$ is maximal, meaning that the two parts of a contrastive cause can only cite different variables if at least one of the models does not contain that variable.

In the case where $M=M'$, this definition is the same as Definition~\ref{defn:bi-factual-cause-simple} because $\Fx=\Fy=\emptyset$ and $\X=\Y$.

\begin{example}
\label{example:extended-arthropod-general-case-contrastive}
Consider now the combination of Examples~\ref{example:arthropod-causal-model} (the simple arthropod classification example) and \ref{example:extended-arthropod-causal-model} (the extended example in which images may come annotated). Let $M$ be the model without the extension and $M'$ be the extended model. Asking why $(M,\u) \models O=Unknown$ and $(M',\u') \models O=Bee$, in which $\u$ and $\u'$ both correspond to features of a Bee but $L=5$ (five legs) and $A=Bee$ in $\u'$, a contrastive cause would be:
\begin{displaymath}
\begin{array}{ll}
   \langle (F_O = f, \emptyset), (F_O = f', A=Bee)\rangle,
\end{array}
\end{displaymath}
in which $f$ and $f'$ refer to the before and after functions for $F_O$ in $M$ and $M'$ respectively, and are hopefully clear from the description. Here, the contrast cites the change in functions and the additional cause $A=Bee$ as the difference condition.
\end{example}

\begin{example}
\label{ex:abstract}
The more general definition is also useful for reasoning about situations in which the fact and surrogate are the same event. That is, ``{Why $(M,\u) \models \phi$ but $(M', \u') \models  \phi$}''? This is useful for situations in which an observer wants to understand why the event $\phi$ still occurs despite the model changing. As an example, consider the two simple structural models in Figure~\ref{fig:abstract-example}, with exogenous variables $U_1$ and $U_2$, and endogenous variables $P$, $Q$, $R$, and $S$. $S$  depends on all four variables in $M$, but there is no variable $Q$ in model $M'$. 

\begin{figure}[!ht]
\begin{subfigure}[b]{0.45\textwidth}
\centering

\begin{tikzpicture}
\begin{scope}[node distance=1.2cm]

\node[state] (S)   {$S$};
\node[state] (R)  [above=of S]  {$R$};
\node[state] (P) [above left=of R]  {$P$};
\node[state] (Q) [above right=of R]  {$Q$};
\node[state] (U1) [above=of P]  {$U_1$};
\node[state] (U2) [above=of Q]  {$U_2$};

\path[every node/.style={sloped,anchor=south,auto=false,font=\tiny}]

(R) edge[]   node[sloped=true]  {$S= 1 - R$}   (S)
(U1)  edge[]   node[]  {$$}   (P)
(U2)  edge[]   node[]  {$~$}   (Q)
(P)  edge[]   node[sloped=false,shift={(+.30in,0in)}]  {$R = max(P,Q)$}   (R)
(Q)  edge[]   node[]  {}   (R)
;
\end{scope}

\end{tikzpicture}
\caption{Structural Model $M$}
\label{fig:abstract-before}
\end{subfigure}
\begin{subfigure}[b]{0.4\textwidth}
\centering
\begin{tikzpicture}
\begin{scope}[node distance=1.2cm]

\node[state] (S)   {$S$};
\node[state] (R)  [above=of S]  {$R$};
\node[state] (P) [above left=of R]  {$P$};
\node[state] (U1) [above=of P]  {$U_1$};

\path[every node/.style={sloped,anchor=south,auto=false,font=\tiny}]

(R) edge[]   node[sloped=true]  {$S= 1 - R$}   (S)
(U1)  edge[]   node[]  {$$}   (P)
(P)  edge[]   node[sloped=true]  {$R = P$}   (R)
;
\end{scope}
\end{tikzpicture}
\caption{Structural Model $M'$}
\end{subfigure}
\caption{Structural Models for for Example~\ref{ex:abstract}}
\label{fig:abstract-example}
\end{figure}

For the contrast between $(M, \u) \models S=0$ and $(M',\u') \models S=0$, in which $\u$ leads to $P=1$, $Q=0$, $R=1$, and $S=0$ and $\u'$ leads to $P=1$, $R=1$, $S=0$, the contrastive cause would be cited as:
\begin{displaymath}
\begin{array}{ll}
 &  \langle (F_R = \max(P, Q), \emptyset), (F_R = P, \emptyset)\rangle.
\end{array}
\end{displaymath}
That is, the difference is in the function $F_R$, not in any of the variables nor the output.
\end{example}

To explore some properties of this, we introduce notation that allows us to reason about changes in models at a meta level. Recall that a structural model $M = (\S, \F)$ consists of a set of signatures $\S$ and a set of functions $\F$. We define the \emph{override} of a set of functions $\F$ by another set $\F'$, denoted $\F \override \F'$ being the same as $\F$, except replacing $F_X$ with $F'_X$ for all variables $X$ such that  $F'_X \in \F$. 
The notation $M \override \F'$ represents the overriding of the functions in $M$ with $\F'$.

\begin{theorem}
\label{thm:BCG-consistent-extended}
If $\mathcal{C}$ comprises \emph{all} counterfactual contrastive actual causes of
$\langle\phi,\psi\rangle$ under respective situations $(M, \u)$ and $(M,\u')$
then for any maximal-consistent subset $\langle (\Fx,\Xex), (\Fy,\Yey)\rangle \subseteq \mathcal{C}$, the following hold:

\begin{enumerate}[(a)]
  \item $(M \override \Fy, s) \models [\Yay]\psi$ 
  \item $(M' \override \Fx, s) \models [\Xax]\phi$.
\end{enumerate}
\end{theorem}
\begin{proof}
The proof for this is an extension of the proof for Theorem~\ref{thm:BC-consistent}. The only case that requires attention is when variables are added/removed to/from the model. In this case, the model $M \override \Fy$ may contain the function $F_X$, which is in $M$ but not $M'$. However, if the variable $X$ is not in $M'$, then $\psi$ cannot refer to it, so its is effectively redundant in $M \override \Fy$.
\end{proof}

\subsection{Presuppositions}

As noted previously, it is difficult to argue that a particular definition of contrastive cause is correct, but we can show our definition behaves according to some commonly-accepted properties. In this section, we show that our definition is consistent with the the idea of contrastive explanation as \emph{presupposed explanation} \cite{lipton1990contrastive}.

\citeA{lipton1990contrastive} notes that to give an explanation for ``{Why \P rather than \Q?}'' is to give a to ``{give a certain type of explanation of \P, \emph{given} \P or \Q, and an explanation that succeeds with the presupposition will not generally succeed without it.''} \cite[p.\ 251]{lipton1990contrastive} (emphasis original). Thus, this states that if we assume that \P and \Q are the only two possible outcomes, and are mutually exclusive, then the actual cause of \P under this assumption will refer to exactly those variables in the difference condition. 

Formally, the assumption is $M \models \phi \xor \psi$ --- that is, under all models of $M$, either $\phi$ is true or $\psi$ is true, and not both. Note the absence of a situation $\u$. Thus, we can re-phrase a counterfactual explanandum as:
\begin{equation}
\label{defn:counterfactual-question-given-P-or-Q}
\textrm{Assuming $M \models (\phi \xor \psi)$, why $(M, \u) \models \phi$~?}
\end{equation}
As a shorthand, we use $M^{\phi\xor\psi}$ to refer to the sub-model of $M$ in which $\phi\xor\psi$ is always true. That is, the functions in $\F$ are restricted such that assignments to all variables always conform to $\phi\xor\psi$.

The set of events $\Xex$ is a \emph{presupposed contrastive cause} of $\phi$ under situation $(M, \u)$ and assumption $M \models \phi \xor \psi$ if and only if the following condition holds:

\begin{description}

\item [PAC] $\Xex$ is an actual cause of $\phi$ under the situation $(M^{\phi \xor \psi}, \u)$.

\end{description}

That is, if we assume that $\phi \xor \psi$ always holds in a structural model, then an \emph{actual} cause of $\phi$ in that model under situation $\u$ is sufficient to identify the different condition.
Note here that the cause is not contrastive because it is not a pair -- it just refers to the variables in $\X$ and their values in $\u$. However, this is enough for us to propose the following theorem.

\begin{theorem}
\label{thm:presupposition-counterfactual}
$\Xex$ is an actual cause of $\phi$ under situation $(M, \u)$ assuming $M \models \phi \xor \psi$
 \emph{if and only if}
 $\langle \Xex, \Xey\rangle$ is a counterfactual contrastive cause of $\langle\phi,\psi\rangle$ under situation $(M, \u)$ for some $\y$.
\end{theorem}
\begin{proof}
This theorem is effectively stating that if AC1-3 hold assuming $\phi \xor \psi$, then CC1-5 hold for some $\y$, and vice-versa. 

The left-to-right case: For CC1, if $\Xex$ is an actual cause under a restricted model $M^{\phi \xor \psi}$, then model $M$ must admit $\Xex$ as (at least) a partial cause for $\phi$. For CC2, $(M, \u) \models \neg\psi$ must hold because $\phi$ holds according to AC1, and $\phi$ and $\psi$ are mutually exclusive. For the remainder, we need to show that a $\y$ exists such that CC3-5 hold. From AC2, we know that there exists some counterfactual situation in which $\phi$ would not have occurred under $M^{\phi\xor\psi}$. In such a situation, it must be that $\psi$ occurred, so all such situations would be candidate values for $\y$. This implies CC3. In addition, the values in $\y$ must make $\psi$ true, and therefore must be different from the values in $\x$, so CC4 holds. Finally, we prove CC5 (maximality) by contradiction. Assume that $\X$ is not maximal. This implies there exists some additional variables $\Y$ not in $\X$ that must change to make $\psi$ hold under $(M,\u)$. However, this would also require these variables to change under $M^{\phi \xor \psi}$, which would mean that $\Xex$ is not a complete actual cause of $\phi$, contradicting the definition of PAC. Therefore, $\X$ must be maximal.

For the right-to-left case, AC1 is implied trivially by CC1: $\Xex$ and $\phi$ hold under $M^{\phi\xor\psi}$, and expanding the model without changing the structural equations themselves will not change the $\phi$. AC2 is implied by CC3: if there is an alternative situation $\u'$ under $M$ such that $\psi$ holds, then that same situation must exist in $M^{\phi\xor\psi}$ because $M^{\phi\xor\psi}$ does not exclude situations in which $\psi$ holds, so any such situation gives us the setting for $\x'$ that is required for the counterfactual situation in AC2.

For AC3, we need to show that the partial cause $\Xex$ under $M$ is minimal under $M^{\phi\xor\psi}$. We prove this by contradiction. Assume that $\Xex$ is not minimal under $M^{\phi\xor\psi}$. This means that there is some variable $W$ that has no effect on $\phi$ under $M^{\phi \xor \psi}$, but is cited as a contrastive cause. Therefore, some part of the contrastive cause cites the events $(\Wew,\W=\z)$ for some $\w,\z$, and that $\Wew$ is a partial cause of $\phi$ under $(M,\u)$ and $\W=\z$ is a partial cause of $\psi$ under the hypothetical situation in CC3. However, $\W=\z$ must then be a counterfactual case for $\W$ that satisfies AC2 under $M^{\phi\xor\psi}$, meaning that it affects $\phi$. This is a contradiction for our assumption that $\Xex$ is not minimal.
\end{proof}

\begin{theorem}
\label{thm:presupposition-bi-factual}
$\Xex$ is an actual cause of $\phi$ under situation $(M, \u)$ assuming $M \models \phi \xor \psi$
and
$\Xey$ is an actual cause of $\psi$ under situation $(M', \u')$ assuming $M \models \phi \xor \psi$
 \emph{if and only if}
 $\langle \Xex, \Xey\rangle$ is a  contrastive bi-factual cause of $\langle\phi, \psi\rangle$ under situations $(M, \u)$ and $(M', \u)$.
\end{theorem}
\begin{proof}
The proof is a straightforward extension of the proof from Theorem~\ref{thm:presupposition-counterfactual}. In brief: the $\y$ referred to in Theorem~\ref{thm:presupposition-counterfactual} is from the surrogate. The two cases on the left of the \emph{if and only if} are symmetric, so the proof above extends to this.
\end{proof}
\section{Constrastive Explanation}
\label{sec:contrastive-explanation}

Now that we have defined contrastive cause, we can define contrastive explanation. This is a simple extension to the existing definition of \citeA{halpern2005causes-part-I}'s definition, but using contrastive causes instead of standard actual causes. 

\subsection{Non-Contrastive Explanation}

In Part~II \cite{halpern2005causes-part-II} of their paper, \citeauthor{halpern2005causes-part-II}  build on the definition of causation from Part~I to provide a definition of \emph{causal explanation}. 
They define the difference between causality and explanation as such: causality is the problem of determining which events cause another, whereas explanation is the problem of providing the necessary information in order to establish causation. Thus, an explanation is a fact that, if found to be true, would be a cause for an explanandum, but is initially unknown. As such, they consider that explanation should be relative to an \emph{epistemic state}. This is in fact a definition of contrastive explanation using \emph{epistemic relevance} \cite{slugoski1993attribution}.

Informally, an explanation is defined in their framework as follows. Consider an agent with an epistemic state $\K$, who seeks an explanation of event $\phi$. A good explanation should: (a) provide more information than is contained in $\K$; (b) update $\K$ in such a way that the person can now understand the cause of $\phi$; and (c) it \emph{may} be a requirement that $\phi$ is true or probable\footnote{In the case of an explainer and explainee, we may say that it is `believed' by the explainer.}.

\citeA{halpern2005causes-part-II} formalise this by defining $\K$ as a set of contexts, which represents the set of `possible worlds' that the questioning agent considers possible. Therefore, an agent believes $\phi$ if and only if $(M,\u) \models \phi$ holds for every $\u$ in its epistemic state $\K$. A complete explanation effectively eliminates possible worlds of the explainee so that they can now determine the cause. Formally, an event $\Xex$ is an \emph{explanation} of event $\phi$ relative to a set of contexts $\K$ if the following hold:

\begin{description}
	
	\item [EX1] $(M, \u) \models \phi$ for each $\u \in \K$ --- that is, the agent believes that $\phi$.
	
	\item [EX2] $\Xex$ is a \emph{sufficient cause} of $\phi$ for all situations $(M, \u)$ where $u \in \K$ such that $(M, \u) \models \Xex$.
	
	\item [EX3] $\X$ is minimal --- no subset of $\X$ satisfies EX2.
	
	\item [EX4] $(M, \u) \models \neg (\Xex)$ for some $\u \in \K$ and $(M, \u') \models \Xex$ for some (other) $\u' \in \K$ --- that is, before the explanation, the agent is initially uncertain whether the information contained in the explanation is true or not, meaning the explanation meaningfully provides information.
	
\end{description}

\begin{example}
\label{example:arthropod-non-contrastive-explanation}
Consider the basic arthropod example (Example~\ref{example:arthropod-causal-model}), in which $O=Unknown$ due to a spider with only 7 legs. The agent knows that the image has 8 eyes and no stinger, but is uncertain of the remaining variables. The explanation for why $O=Unknown$ is just $L=7$ (7 legs). This is a sufficient cause for $O=Unknown$, is minimal, and the agent does not know it previously.

For the extended arthropod example, consider the same case, but with $V=Pass$ (known to the agent) and $A=Unknown$ (unknown to the agent). An explanation for why $V=Pass$ would cite the pair $(O=Unknown,A=Unknown)$. The agent would need to know both parts of information to determine the cause. Another explanation would be $(L=7,A=Unknown)$, as knowing $L=7$ allows the agent to determine $O=Unknown$. If the agent already knows $O=Unknown$, then the explanation is a singleton again; either $L=7$ or $O=Unknown$ will suffice.
\end{example}

\subsection{Contrastive Counterfactual Explanation}

We extend the above definition to contrastive counterfactual causes. As with the Halpern and Pearl definition, it is defined relative to an epistemic state and model, however, as it describes a contrastive cause, the explanation is a pair.

\begin{definition}[Contrastive counterfactual Explanation]
	\label{defn:counterfactual-explanation-basic}
Given a structural model $M$, a pair of events $\langle \Xex, \Xey\rangle$ is a \emph{contrastive counterfactual explanation} of $\langle\phi,\psi\rangle$ relative to $\K$ if and only if the following hold:

\begin{description}
 
 \item [CE1] $(M, \u) \models \phi \land \neg\psi$ for each $\u \in \K$ --- that is, the agent accepts that $\phi$ and that $\neg\psi$.

 \item [CE2] $\langle \Xex, \Xey\rangle$ is a sufficient counterfactual cause for $\langle\phi,\psi\rangle$, for each $\u\in\K$ such that $(M,\u) \models \Xex$.

 \item [CE3] $\X$ is minimal --- no subset of $\X$ satisfies CE2.

 \item [CE4] $(M, \u) \models \neg (\Xex)$ for some $\u \in \K$ and $(M, \u') \models \Xex$ for some (other) $\u' \in \K$; and for some $\W=\w$ such that $\w\neq\x$, $(M_{\Waw}, \u) \models \Xey$ for some $\u \in \K$ and $(M_{\Waw}, \u')\models \neg(\Xey)$ for some (other) $\u' \in \K$ -- that is, agent is initially uncertain whether the explanation is true or not, meaning the explanation provides meaningful information.

\end{description}
\end{definition}

\begin{example}
Consider the same two cases from Example~\ref{example:arthropod-non-contrastive-explanation}. An explanation for why $O=Unknown$ rather than $O=Spider$ would cite the pair $\langle L=7, L=8\rangle$: the image has 7 legs but requires 8 to be a spider. We can already see that this is more informative than the non-contrastive cause, because we are given the counterfactual case of what \emph{should have been} to make $O=Spider$.

For the extended case, an explanation for why $V=Pass$ rather than $V=Fail$ (the only possible foil) is the pair of tuples $\langle (O=Unknown, A=Unknown), (O=X,A=X)\rangle$, where $X$ is one of $Spider$, $Beetle$, etc., or the pair of tuple $\langle (L=7, A=Unknown), (L=8,A=Spider)\rangle$, and similarly for other types. Again, if the agent already knows $A$ or $L$, then pairs of singletons suffice. 
\end{example}

Definition~\ref{defn:counterfactual-explanation-basic} defines a counterfactual  contrastive explanation as finding part of a counterfactual cause that satisfies the conditions CE1-4. However, we can think of this in different way: finding partial explanations for each of $\phi$ and $\psi$ and taking the difference between these, where we define a \emph{partial explanation} as just a subset of an explanation.

\begin{definition}[Contrastive Counterfactual Explanation -- Alternative Definition]
	Given a structural model $M$, a pair of events $\langle \Xex, \Xey\rangle$ is a \emph{contrastive counterfactual explanation} of $\langle\phi,\psi\rangle$ relative to $\K$ if and only if the following hold:

	\begin{description}
		
		\item [CE1$'$] $\Xex$ is a partial explanation of $\phi$ in $(M,\u)$.
		
		\item [CE2$'$] There is a non-empty set $\W \subseteq \V$ and a setting $\w$ of variables in $\W$ such that $\Xey$ is a partial explanation of $\psi$ under situation $(M_{\Waw}, \u)$.
		
		\item [CE3$'$] $(\Xex) \cap (\Xey) = \emptyset$ --- the difference condition.
		
		\item [CE4$'$] $\X$ is maximal --- that is, there is no superset of $\X$ that satisfies CE1$'$-3$'$.
		
	\end{description}
\end{definition}

\begin{theorem}
	\label{thm:CE-counterfactual-consistent}
	CE1-4 \emph{iff} CE1-4$'$ --- that is, the two definitions of contrastive counterfactual explanation are equivalent.
\end{theorem}
\begin{proof}
	Left-to-right case: 
	(CE1$'$)  This holds from CE2-4. If $\langle \Xex, \Xey\rangle$ is a sufficient counterfactual cause for $\langle\phi, \psi\rangle$ that is minimal and uncertain, then $\Xex$ must be a partial explanation of $\phi$ under $(M,\u$); that is, the agent believes $\phi$, some superset of $\Xex$ is an actual cause of $\phi$, and the agent is uncertain about some of that superset. 
	(CE2$'$)  The same argument holds, except that $\Xey$ is true under the hypothetical situation $(M_{\Waw},\u)$ from CE2; and therefore, this hypothetical situation is a witness for CE2$'$. 
	(CE3$'$) 
	The difference condition holds because this is a requirement of CE2. 
	(CE4$'$) holds from the maximality condition in CE2.
	This establishes the left-to-right case.
	
	Right-to-left case:
	(CE1) This holds directly from CE1$'$, because the acceptance of $\phi$ in $\K$ is a condition of an explanation under the original \citeauthor{halpern2005causes-part-II} definition, and $\psi$ must be false whenever $\phi$ is true. 
	(CE2) If $\Xex$ and $\Xey$ are partial explanations of $\phi$ under $(M,\u)$ and $\psi$ under $(M_{\Waw},\u)$ respectively, then they must be partial causes too. If their intersection is empty and $\X$ is maximal, then this defines a sufficient cause, so CE2 holds.
	(CE3) 
	We prove this via contradiction. Assume $\X$ is not minimal. This implies that there is some strict  superset $\Y \supset \X$ that satisfies CE1$'$-4$'$ and CE2. However, if this were the case, then CE4$'$ would not hold: $\X$ would not be maximal over CE1$'$-3$'$, which is a contradiction, so our assumption is false.
	(CE4)
	If $\Xex$ and $\Xey$ are partial explanations under $(M,\u)$ and some hypothetical counterfactual case $(M_{\Waw},\u)$, then the agent must be uncertain about $\Xex$ in $(M,\u)$ and $\Xey$ in $(M_{\Waw},\u)$.
	This establishes the right-to-left case, and the theorem holds.
\end{proof}

\subsection{Bi-factual Contrastive Explanation}

For the bi-factual case, an explanation is similar, however, it refers to two epistemic states, $\K$ and $\K'$, in which $\K$ models the uncertainty of the individual in the situation $(M,\u)$ and $\K'$ models the uncertainty in $(M',\u')$.

\begin{definition}[Contrastive Bi-factual Explanation -- Simple Case]
\label{defn:bi-factual-explanation-basic}
Given a structural model $M$, a pair $\langle \Xex, \Xey\rangle$ is a \emph{contrastive bi-factual explanation} of $\langle\phi,\psi\rangle$ relative to two epistemic states $\K$ and $\K'$ if and only if the following hold:

\begin{description}
 
 \item [BE1] $(M, \u) \models \phi$ for each $\u \in \K$ and  $(M', \u') \models \psi$ for each $\u' \in \K'$ --- that is, the agent accepts that $\phi$ under $(M,\u)$ and that $\neg\psi$ under ($M',\u')$.

 \item [BE2] for each $\u\in\K$ such that $(M,\u) \models \Xex$ and $\u'\in\K'$ such that  $(M',\u') \models \Xey$, $\langle \Xex, \Xey\rangle$ is a sufficient bi-factual cause for $\langle \phi, \psi\rangle$ under $(M,\u)$ and $(M',\u')$.

 \item [BE3] $\X$ is minimal ---  that is, no superset of $\X$ satisfies BE2.

 \item [BE4] $(M, \u) \models \neg (\Xex)$ for some $\u \in \K$ and $(M, \u') \models \Xex $ for some (other) $\u' \in \K$; and  $(M, \u) \models \Xey$ for some $\u \in \K'$ and $(M, \u') \models \neg(\Xey)$ for some (other) $\u' \in \K'$ -- that is, the agent is initially uncertain whether the explanation is true or not, meaning the explanation provides  meaningful information.

\end{description}
\end{definition}

This is similar to the definition of CE, except that the rules refer to an actual situation $\u'$, rather than the hypothetical situation implied by CE2.
The more general case in which there are differences between the models is straightforward projection of this.

\begin{example}
Consider the case of the 7-legged spider (situation $\u_7$), and a second case of a `proper' spider ($\u_8$). The agent is uncertain of all variables and asks why $O=Unknown$ under $(M,\u_7)$ and $O=Spider$ under $(M,\u_8)$. The explanation is as before: $\langle L=7, L=8\rangle$. Note here that the agent already knows that $L=8$, because it knows that $O=Spider$, so can determine the values of the input variables. However, we still cite this in the explanation because it contrasts $L=7$. The extended case is similar to the counterfactual explanation.
\end{example}

Definition~\ref{defn:bi-factual-explanation-basic} defines bi-factual explanation as finding part of a counterfactual cause that satisfies the conditions BE1-4. However, we can think of bi-factual explanation in different way: finding partial explanations for each of $\phi$ and $\psi$ and taking the difference between these, where we define \emph{partial explanation} as just subsets of explanations.

\begin{definition}[Contrastive Bi-factual Explanation -- Simple Case, Counterfactual Definition]
Given a structural model $M$, a pair of events $\langle \Xex, \Xey\rangle$ is a \emph{contrastive bi-factual explanation} of $\langle\phi,\psi\rangle$ relative to two epistemic states $\K$ and $\K'$ if and only if the following hold:

\begin{description}
 
 \item [BE1$'$] $\Xex$ is a partial explanation of $\phi$ in $(M,\u)$.

 \item [BE2$'$] $\Xey$ is a partial explanation of $\psi$ in $(M,\u')$.

 \item [BE3$'$] $(\Xex) \cap (\Xey) = \emptyset$ --- the difference condition.

 \item [BE4$'$] $\X$ is maximal --- that is, there is no superset of $\X$ that satisfies BE1$'$-3$'$.

\end{description}
\end{definition}

\begin{theorem}
BE1-4 iff BE1$'$-4$'$ --- that is, the two definitions of contrastive bi-factual explanation are equivalent.
\end{theorem}
\begin{proof}
	The proof for this is similar to the proof for Theorem~\ref{thm:CE-counterfactual-consistent}, except simpler because we deal only with factual situations and no hypothetical situations.
\end{proof}

\subsection{Non-Contrastive General Explanation}

The definitions provided in the previous section merely allow explanations in which the causal model is known to the explainee agent, but the agent is uncertain which context is the real context. A more general definition allows for explanations in which the agent is also uncertain about the causal model, and thus the explanation is about both the causal model and the context.

\citeA{halpern2005causes-part-II} present an extended definition of explanation based on this idea.
In this case, an epistemic state $\K$ is now a set of situations $(M, \u)$ instead of a set of just contexts. A general explanation is of the form  $(\alpha, \Xex)$, in which $\alpha$ is a causal formula. The first component restricts the set of models, while the second restricts the set of contexts.

A formula-event pair $(\alpha, \Xex)$ is an explanation of event $\phi$ relative to a set of situations $\K$ if:

\begin{description}
	
	\item [EX1] $(M, \u) \models \phi$ for each $(M,\u) \in \K$ (unchanged).
	
	\item [EX2] for all situations $(M, \u)$ such that $(M, \u) \models \Xex$ and $M \models \alpha$ ($\alpha$ is valid in all contexts consistent with $M$), $\Xex$ is a \emph{sufficient cause} of $\phi$.
	
	\item [EX3] $(\alpha,\Xex)$ is minimal --- there is no pair $(\alpha',\Xex[']) \neq (\alpha,\Xex)$ satisfying EX2 such that $\{M'' \in M(\K) \mid M'' \models \alpha'\} \supseteq \{M'' \in M(\K) \mid M'' \models \alpha\}$ and $\Xex['] \subseteq \Xex$, where $M(\K) = \{M \mid (M,\u) \in \K \textrm{ for some } \u\}$.
	
	\item [EX4] $(M, \u) \models \neg (\Xex)$ for some $(M, \u) \in \K$ and $(M', \u) \models \Xex$ for some (other) $(M', \u') \in \K$ --- that is, the agent is uncertain, as before.
	
\end{description}

In this definition, the two parts of the explanation play different roles. The formula $\alpha$ characterises the part of the model that is unknown to the agent to just enough information to understand the causes of $\phi$; while $\phi$ is an explanation in that restricted set of models.

\begin{example}
	As a simple example, consider an agent who does not know how the arthropod system works at all, and confronted with $O=Spider$, they ask why. An explanation is the pair:
	\[
	(L=8 \land S=\yes \land E=8 \land C=\yes \land W=0 \Rightarrow O=Spider, L=8)
	\]
	plus one for all other variables other than $L$. 
	The formula informs the explainee what the properties of a spider are, but does not need to define the entire model nor even the properties of other arthropods.
	
	However, the $\alpha$ part of the explanation can be arbitrary causal formula. For example, given a 7-legged spider with no annotation, which will cause $V=Pass$, an explanation could refer to formula such as:
	\[ 
	(O=Unknown \land A=Unknown) \Rightarrow [A \leftarrow Spider] (V=Pass),
	\]
	 which means that when both variables are unknown, adding an annotation will still give a result of $Pass$.
\end{example}

\subsection{General Contrastive Explanation}

The more general case of contrastive explanation is straightforward to project from this definition. We give just the definition for bi-factual explanation.

\begin{definition}{General Contrastive Bi-factual Explanation}
	\label{defn:general-contrastive-bi-factual-explanation}
Given a structural model $M$, a pair of formula-event pairs $\langle (\alpha, \Xex), (\beta, \Xey)\rangle$ is a \emph{general contrastive bi-factual explanation} of $\langle\phi,\psi\rangle$ relative to two epistemic states $\K$ and $\K'$ if and only if the following hold:

\begin{description}
	
	\item [BE1\G] $(M, \u) \models \phi$ for each $(M,\u) \in \K$ and  $(M', \u') \models \psi$ for each $(M,\u') \in \K'$.
	
	\item [BE2\G] for all situations $(M, \u)$ such that $(M, \u) \models \Xex$ and $M \models \alpha$ and all situations $(M',\u')$ such that $(M',\u') \models \Xey$ and $M' \models \beta$, $\langle (\alpha, \Xex), (\beta, \Xey)\rangle$ is a \emph{sufficient bi-factual cause} of $\langle\phi,\psi\rangle$.
	
	\item [BE3\G] $(\alpha,\Xex,\beta,\Xey)$ is minimal ---  there is no tuple $(\alpha',\Xex['],\beta',\Xey[']) \neq (\alpha,\Xex,\beta,\Xey)$ satisfying BE2\G such that $\{M'' \in M(\K) \mid M'' \models \alpha'\} \supseteq \{M'' \in M(\K) \mid M'' \models \alpha\}$, similarly for $\beta$, $\Xex['] \subseteq \Xex$, and $\Xey['] \subseteq \Xey$.
	
	\item [BE4\G] $(M, \u) \models \neg (\Xex)$ for some $(M,\u) \in \K$ and $(M', \u') \models \Xex $ for some (other) $(M',\u') \in \K$; and  $(M, \u) \models \Xey$ for some $(M,\u) \in \K'$ and $(M, \u') \models \neg(\Xey)$ for some (other) $(M,\u') \in \K'$ -- that is, the agent is initially uncertain whether the explanation is true or not.
	
\end{description}
\end{definition}

The general counterfactual explanation case is straightforward to extend from Definition~\ref{defn:general-contrastive-bi-factual-explanation}, however, one important point of difference is that the explanation is not a pair, but a triple, $\langle \alpha, \Xey, \Xey)$. There is no requirement for the second formula $\beta$ because there is only one model to characterise.

\begin{example}
Consider the extended arthropod system in the situation in Example~\ref{example:arthropod-non-contrastive-explanation}, where there is an image of a spider with 7 legs. In this case, the verification passes because $O=Unknown$. The agent knows all variables but is unaware of $F_V$, so does not know the verification procedure, and asks ``Why $V=Pass$ instead of $V=Fail$?''

In this case, the explanation is a formula expressing the semantics of $F_V$, and no variables:
\[
\langle (O=Unknown \lor A=Unknown \Rightarrow V = Pass), L=7, L=8 \rangle 
\]

\end{example}

\subsection{Example: Goal-Directed AI Planning}

Throughout the paper, we have used the two examples of the arthropod system to illustrate ideas. In this section, we consider a different type of AI system: goal-directed planning.

\begin{example}
	\label{example:planning-counterfactual}
	Consider an abstract example of a goal-directed planning system that needs to choose which actions $A_1$, $A_2$, and $A_3$ to apply. Using a simple action language, we define these actions, their preconditions, and their effects as:
   
   \begin{center}
   \begin{tabular}{llll}
   \toprule
      \textbf{Action} & \textbf{Pre} & & \textbf{Effect}\\
      \midrule
      $A_1$ & $P_1$ & $\rightarrow$ & $G_1 \land G_3$\\
      $A_2$ & $P_2$ & $\rightarrow$ & $G_2 \land G_3$\\
      $A_3$ & $true$  & $\rightarrow$ & $P_2$\\
      \bottomrule
    \end{tabular}
    \end{center}
    in which $A_{[1-3]}$ are names of the actions, $G_{[1-3]}$ are propositions modelling goals, and $P_{[1-2]}$ are propositions modelling action preconditions. The planner can apply none, one, or many actions. 
    
    Figure~\ref{fig:planning-causal-graph} shows the causal graph for this, in which $U_{[1-5]}$ are exogenous variables. Variables are Boolean. The structured equations are such that action $A_1$ is selected if $G_1$ or $G_3$ is the goal, and its precondition $P_1$ holds; $A_2$ is selected if $G_2$ or $G_3$ is the goal, and its precondition $P_2$ holds; and $A_3$ is selected if precondition $P_2$ needs to be made true. Note that this does not model the cause of the preconditions  goals becoming true/false, but the \emph{cause} of action selection, which makes the graph appear somewhat inverted. The parent node for each action has both the variables it requires to be true to execute the action as well as the variables the action will change; e.g.\ $A_1$ will be `fired' if $P_1$ is true and $G_1$ is true,  which counter-intuitively models that in the actual planning problem, the goal is currently false and should \emph{become} true. We could also add a node which states whether the goal is true/false and only execute the action if the goal is false, but we  omit this for simplicity. Note that $P_2$ is the parent of $A_3$, modelling that this is $A_3$'s intermediate `goal' -- it makes $P_2$ true, thus enabling $A_2$ to be selected next.
	
	\begin{figure}[!th]
		\centering
		\centering
\begin{tikzpicture}
\begin{scope}[node distance=1.2cm]
\tikzstyle{empty state}=[fill=none,inner sep=0pt,minimum size=2.5mm,opacity=0.0]

\node[empty state] (B)  {};
\node[state] (A1) [left=of B] {$A_1$};
\node[state] (A2) [right=of B] {$A_2$};
\node[state] (P1) [above left=of A1] {$P_1$};
\node[state] (P2) [above right=of A2] {$P_2$};
\node[state] (G1) [above=of A1]  {$G_1$};
\node[state] (G3) [above=of B]  {$G_3$};
\node[state] (G2) [above=of A2]  {$G_2$};
\node[state] (U1) [above=of P1]  {$U_1$};
\node[state] (U2) [above=of G1]  {$U_2$};
\node[state] (U3) [above=of G3]  {$U_3$};
\node[state] (U4) [above=of G2]  {$U_4$};
\node[state] (U5) [above=of P2]  {$U_5$};

\node[state] (A3) [below right=of P2] {$A_3$};

\path[every node/.style={sloped,anchor=south,auto=false,font=\tiny}]

(U1) edge[]   node[sloped=true]  {$$}   (P1)
(U2) edge[]   node[sloped=true]  {$$}   (G1)
(U3) edge[]   node[sloped=true]  {$$}   (G3)
(U4) edge[]   node[sloped=true]  {$$}   (G2)
(U5) edge[]   node[sloped=true]  {$$}   (P2)
(P1) edge[]   node[sloped=true]  {$$}   (A1)
(G1) edge[]   node[sloped=true]  {$$}   (A1)
(G3) edge[]   node[sloped=true]  {$$}   (A1)
(G3) edge[]   node[sloped=true]  {$$}   (A2)
(G2) edge[]   node[sloped=true]  {$$}   (A2)
(P2) edge[]   node[sloped=true]  {$$}   (A2)
(P2) edge[]   node[sloped=true]  {$$}   (A3)

;
\end{scope}

\end{tikzpicture}
		\caption{Causal graph for goal-directed planning}
		\label{fig:planning-causal-graph}
	\end{figure}
	
	Now consider the case in which $G_1$ and $G_3$ are the goals (while $G_2$ is false) and $P_1$ and $P_2$ are both true, implying that $A_1$ is true and $A_2$ is false. A contrastive question could be: ``Why $A_1$ rather than $A_2$?", which would be modelled as ``Why $(M, \u) \models A_1 \land A_2$ rather than $\neg A_1 \land A_2$?''.	$\K$ (the epistemic state of the explainee) is  such that $G_1$ is known to be true, but the agent is unsure of the other goals and the preconditions.
	
	The contrastive counterfactual cause for this is the pair $\langle (G_1, \neg G_1), (\neg G_2, G_2)\rangle$. That is, for the $A_2$ to be true instead of $A_1$ (CC3), it would require that the goals $G_1$ and $G_2$ are swapped. CC1-2 hold trivially, and CC4 (the difference conditions) holds because there are no common events. CC5 (maximality) holds because changing the values of $G_1$ or the preconditions $P_1$ and $P_2$ do not satisfy the difference condition CC4.
	
	The contrastive \emph{explanation}, however, consists only of  $\langle(\neg G_2, G_2)\rangle$ -- the agent already knows that $G_1$ is true so including $G_1$ would not satisfy both CE3 (the minimality condition) and CE4 (the `meaningful' condition).
\end{example}

\begin{example}
	Consider a bi-factual setting with two situations $\u_1$ and $\u_2$. In both situations, $G_3$ is the only goal. In $\u_1$, precondition $P_1$ is true while $P_2$ is false, and vice-versa for $\u_2$. The explainee agent knows only that action $A_1$ was selected under $\u_1$ and $A_2$ was selected under $\u_2$. The bi-factual explanation for this is $\langle (P_1, \neg P_2), (\neg P_1, P_2)\rangle$. The goals are not included even though the agent does not know their values, because they are the same between the two situations, so do not satisfy the difference condition.
\end{example}

\begin{example}
Finally, consider the example of $A_3$ being selected in order to make $P_2$ true and allow $A_2$ to be selected in the next time step. The goal is $G_2$ and the explainee knows the values of all goal variables and action variables, does not know the values of the preconditions, and asks why $A_3$ rather than $A_2$ . 

The effect of CE4 is that this has no explanation! Intuitively, one may expect that $(\neg P_2, P_2)$ to be offered, however for this to be an explanation, CE4 requires  that there is some situation $\u \in \K$ in which $P_2$ could be true. But this is not possible because the agent knows that $A_2$ is false and $G_2$ is true, which cannot be the case if $P_2$ is true, so no such situation exists. According to the model $M$, there can be only situation in $\K$ where the goals are all known as $\neg G_1$, $G_2$, and $\neg G_3$, and in that situation $\neg P_1$ and $P_2$ hold. This offers the agent a complete explanation already. This makes sense: the agent does not require an explanation because it can infer the values of $P_1$ and $P_2$ itself.

However, consider the case of a general contrastive explanation in which the agent's knowledge is missing part of the structure of the causal graph; specifically, that $P_2$ is the precondition of $A_2$, meaning that the edge $P_2 \rightarrow A_2$ is missing from the graph in Figure~\ref{fig:planning-causal-graph}. Now we have an explanation! In this case, the explanation is $\langle (F_{A_2} = f, \emptyset), (F_{A_2} =  f', \neg P_2)\rangle$, in which $f$ is the definition of $F_{A_2}$ without the precondition, and $f'$ includes the precondition.
\end{example}

\section{Conclusion}
\label{sec:conc}

Using structural causal models,
\citeA{halpern2005causes-part-II} define explanation as a fact that, if found to be true, would constitute an actual cause of a specific event. In this paper, we extend this definition of explanation to consider \emph{contrastive explanations}. Founded on existing research in philosophy and cognitive science, we define two types of contrastive why-questions: counterfactual why--questions (`rather than') and and bi-factual why--questions (`but'). We define `contrastive cause' for these two questions and from this, build a model of contrastive explanation. We show that this model is consistent with well-accepted properties of contrastive explanation, and with alternative definitions.

The aim of this work is to provide a general model of contrastive explanation. While there are many examples of researchers considering counterfactual contrastive questions in explainable artificial intelligence, few consider bi-factual questions. Even fewer exploit the power of the difference condition, instead providing two full explanations: one of the fact and one of the foil. In essence, they consider contrastive questions but not contrastive explanations. The difference condition is what brings power and relevance to contrastive explanations, and as such, giving two complete explanations does not correctly answer the question. We hope that this article serves as a basis for researchers in explainable artificial intelligence to adopt the idea of the difference condition and ultimately give better explanations to people.

\bibliographystyle{plainnat}
\bibliography{explanation,other}


\end{document}